\documentclass[dvipsnames]{article}
\newcommand{\conferencetemplate}{icml2022}

\newcommand{\isdraft}{false}
\newcommand{\revertchanges}{false}
\newcommand{\editmode}{false}

\usepackage{draftwatermark}

\usepackage{harwiltzmath}

\usepackage{microtype}
\usepackage{graphicx}
\usepackage{subfigure}
\usepackage{booktabs} 

\usepackage{hyperref}

\usepackage[accepted]{icml2022/icml2022}

\usepackage{amsmath}
\usepackage{amssymb}
\usepackage{mathtools}
\usepackage{amsthm}

\newtheorem{corollary}{Corollary}
\newtheorem{lemma}{Lemma}
\newtheorem{theorem}{Theorem}
\newtheorem{assumption}{Assumption}
\newtheorem{definition}{Definition}
\newtheorem{proposition}{Proposition}
\newtheorem{remark}{Remark}

\usepackage{natbib}
\usepackage{enumitem}

\DeclareMathOperator{\cbrprocess}{\overleftarrow{G}}

\usepackage{thmtools}
\usepackage{thm-restate}

\newcommand{\xxX}{\mathcal{X}}               
\newcommand{\xxA}{\mathcal{A}}               
\newcommand{\xxf}[1][\pi]{\mu_{#1}}            
\newcommand{\xxs}[1][\pi]{\pmb{\sigma}_{#1}} 
\newcommand{\xxF}{\mathcal{F}}               
\newcommand{\xxL}{\mathscr{L}}               
\newcommand{\xxG}[2][\pi]{G^{#1}(#2)}              
\newcommand{\xxW}{W}                         
\newcommand{\xxStatsSet}[1]{\mathcal{S}_{#1}} 
\newcommand{\xxLoss}{\mathscr{F}}            
\newcommand{\xxEntropy}{\mathcal{H}}         
\newcommand{\xxTruncRet}{\overline{G}}       
\newcommand{\xxJointFD}{J}                   
\newcommand{\xxJointFDCB}{\overleftarrow{\xxJointFD}} 
\newcommand{\xxMMSIterate}{\widetilde{\returnmeasure}} 
\newcommand{\xxPSpaceSample}{\Omega}         
\newcommand{\xxPSpaceF}{\xxF}                
\newcommand{\xxPSpaceFiltration}{\indexedabove{t}{\xxF}} 
\newcommand{\xxPSpaceMeasure}{\mathsf{P}}    
\newcommand{\xxPSpace}{(\xxPSpaceSample, \xxPSpaceF,
  \xxPSpaceFiltration, \xxPSpaceMeasure)}    

\DeclareMathOperator*{\expect}{{\huge \mathbf{E}}}

\newcommand{\argthing}[2][nothing]{\ifthenelse{ \equal{#1}{nothing} }{\text{arg #2}}{\underset{#1}{\text{arg #2}}}}

\newcommand{\argmax}[1][nothing]{\argthing[#1]{max}}
\newcommand{\argmin}[1][nothing]{\argthing[#1]{min}}
\newcommand{\Borel}{\mathscr{B}}
\newcommand{\fdneighbor}[1][\epsilon]{\mathsf{A}_{#1}}

\renewcommand{\epsilon}{\varepsilon}

\usepackage{mathtools}
\allowdisplaybreaks

\usepackage{trackchanges}
\newcommand{\changecolor}{black}
\newcommand{\revertchangecolor}{black}
\newcommand{\bigchange}[3][Anon]{
  \ifthenelse{ \equal{\revertchanges}{false} }{
    {\color{\changecolor}${}^{\text{#1:}}$ #3}
  }{
    {\color{\revertchangecolor} #2}
  }}

\addeditor{Marc}
\addeditor{Harley}
\newcommand{\editornote}[2]{\ifthenelse{\equal{\editmode}{true}}{\note[#1]{#2}}{}}
\newcommand{\marc}[1]{\editornote{Marc}{#1}}
\newcommand{\harley}[1]{\editornote{Harley}{#1}}

\newcommand{\statediffuse}[1][\Phi]{\mathbf{K}^x_{#1}}
\newcommand{\statsdiffuse}[1][\Phi]{\mathbf{K}^s_{#1}}

\renewcommand{\usequantumnotation}{false}
\begin{document}
\twocolumn[
\icmltitle{Distributional Hamilton-Jacobi-Bellman Equations for Continuous-Time Reinforcement Learning}
\begin{icmlauthorlist}
\icmlauthor{Harley Wiltzer}{mcgill,mila}
\icmlauthor{David Meger}{mcgill}
\icmlauthor{Marc G. Bellemare}{mila,google,cifar}
\end{icmlauthorlist}
\icmlaffiliation{mcgill}{McGill University, Montreal, Canada}
\icmlaffiliation{mila}{Mila -- Quebec AI Institute}
\icmlaffiliation{google}{Google Brain, Montreal, Canada}
\icmlaffiliation{cifar}{CIFAR Fellow}

\icmlcorrespondingauthor{Harley Wiltzer}{harley.wiltzer@mail.mcgill.ca}
\vskip 0.3in
]
\ifthenelse{\equal{\isdraft}{true}}{%
    \rhead{page \thepage}
}{%
    \SetWatermarkText{}
}
\printAffiliationsAndNotice{}
\newcommand{\ifconfelse}[3]{\ifthenelse{\equal{\conferencetemplate}{#1}}{#2}{#3}}
\newcommand{\ifconf}[2]{\ifconfelse{#1}{#2}{}}
\begin{abstract}

Continuous-time reinforcement learning offers an appealing formalism for describing control problems in which the passage of time is not naturally divided into discrete increments.
Here we consider the problem of predicting the distribution of returns obtained by an agent interacting in a continuous-time, stochastic environment. 
Accurate return predictions have proven useful for determining optimal policies for risk-sensitive control, learning state representations, multiagent coordination, and more.
We begin by establishing the distributional analogue of the Hamilton-Jacobi-Bellman (HJB) equation for It\^o diffusions and the broader class of Feller-Dynkin processes.\marc{We may be able to just add the max to your HJB. Let's discuss.}
We then specialize this equation to the setting in which the return distribution is approximated by $N$ uniformly-weighted particles, a common design choice in distributional algorithms.
Our derivation highlights additional terms due to \emph{statistical diffusivity} which arise from the proper handling of distributions in the continuous-time setting. Based on this, we propose a tractable algorithm for approximately solving the distributional HJB based on a JKO scheme, which can be implemented in an online control algorithm. We demonstrate the effectiveness of such an algorithm in a synthetic control problem.

\end{abstract}
\section{Introduction}\label{s:intro}


In continuous-time reinforcement learning \citep{munos1997convergent,munos1997reinforcement},
the expected return or \emph{value function} is characterized by a
partial differential equation (PDE) known as the
Hamilton-Jacobi-Bellman (HJB) equation \citep{Krylov1980ControlledDP,fleming2006controlled}.
This equation can be solved using numerical methods \citep{Munos2004ASO}, producing a policy that is optimal in the sense it maximises the expected return and avoids the error and computational costs associated with discretizing time.



This paper presents an analysis of the behavior of the \emph{distribution} over
returns in the continuous-time limit, as opposed to solely its
expectation.
Existing literature in \emph{distributional} reinforcement learning
(DRL) has demonstrated that modeling return distributions aids the
policy learning process, even when decisions are based only on the
expectations of the return distributions \citep{Bellemare2017ADP,
  hessel2018rainbow, Rowland48495}. Beyond that, statistics of the
return distributions may provide useful signals for exploration
\citep{mavrin2019distributional} and risk-sensitive behavior
\citep{prashanth2013actor,chow14algorithms,tamar15optimizing,dabney2018implicit, NEURIPS2019_f471223d,
  halperin2021distributional,prashanth21risksensitive}.

\textbf{A distributional HJB equation.} We first establish 
the distributional analogue to the HJB equation for a broad class of continuous-time environments, when the policy is fixed (the \emph{policy evaluation} setting). Because return distribution functions are infinite-dimensional objects (both in state and return), they are in general quite complex. However, we obtain a concise form of the distributional HJB by appealing to the notion of an infinitesimal generator \citep{rogers1994diffusions}, specifically applied to the cumulative distribution function of the return distribution.
This basic result extends to the expected-return control setting by obtaining an optimal policy from the usual HJB equation and subsequently solving the distributional HJB equation with this policy.

\textbf{Specialization to finitely-supported distributions.} In distributional RL, it is common to represent return distributions parametrically, for example with a finite collection of Dirac deltas. With care, this makes it possible to derive practical algorithms that find finite-memory approximations to the return distribution function.
Our second contribution is to specialize the distributional HJB equation to finite collections of statistical functionals and subsequently to what \citet{bdr2022} call the \emph{quantile probability representation}. The result is effectively a set of HJB equations and associated distributional constraints, one per parameter.

\textbf{Finite-difference algorithm for continuous-time distributional RL.} Finally, we extend the algorithm of \citet{munos1997reinforcement} for optimal control of continuous-time environments to the distributional setting. In particular, the inner loop of our algorithm involves finding distributional approximations by means of a JKO scheme previously employed by \citet{martin2020stochastically}.
Effectively, this method solves the quantile HJB equation at the desired level of accuracy, without explicitly discretizing time. In a synthetic experiment, we find that our technique produces far fewer artefacts than the equivalent discrete-time method.
\section{Setting}\label{s:background}

%

In this section we establish the mathematical framework that enables us to characterize the random return. To describe a continuous-time environment, we use the formalisms of Feller-Dynkin processes and \Ito\ diffusions. This is sufficient to establish the general distributional HJB equation; to derive more practical equations, however, we must also introduce notions from statistical functional theory.

Let $\mathscr{P}=\xxPSpace$ be a filtered probability space. The notation $\probset[p]{A}$ refers to the space of all probability measures with bounded $p$-moments, and $\probset{A}\equiv\probset[1]{A}$. Moreover, we denote by $\hessian{x}, \jacobian_x$ the Hessian and Jacobian operators, taken with respect to the $x$ variable.

\subsection{Continuous-time reinforcement learning}

We consider a continuous-time Markov decision process with a compact state space $\xxX \subseteq \R^d$ and a discrete action space $\xxA$.
The state and action processes are respectively $\indexedabove{t}{X} :\R_+\to\xxX$ and $\indexedabove{t}{A}:\R_+\to\xxA$. The actions are determined by a stochastic policy $\pi:\xxX\to\probset{\xxA}$ such that $A_t\sim\pi(\cdot\mid X_t)$. For a fixed policy, the state process is assumed to be a Feller-Dynkin process \citep{rogers1994diffusions} with transition semigroup $\indexedabove{t}{P^\pi}$; a primer on Feller-Dynkin processes and other continuous-time objects is given in Appendix
\ref{app:ctmp}. Finally, $r: \xxX \to \R$ is a bounded reward function.


When the agent exits the interior $\mathcal{O}$ of $\xxX$, we say that the process has stopped or \emph{terminated}, and no further rewards are earned.\footnote{A state $x \in \xxX \setminus \mathcal{O}$ is said to be \emph{terminal}.} We will assume that $\mathcal{O}$ is Borel-measurable. The agent's (random) exit time $T$ from $\mathcal{O}$, expressed as
\begin{align*}
    T = \inf\left\{t\in\R_+ : X_t\not\in\mathcal{O}\right\}
\end{align*}
is a stopping time with respect to the canonical filtration \citep{le2016brownian}.

The (discounted) return $\xxG{x}$ from state $x\in\xxX$ is the reward accumulated by following policy 
$\pi$ starting at state $x$, with rewards discounted exponentially in time by
$\gamma\in(0,1)$:
\begin{align}
  \xxG{x} &\triangleq \int_0^T\gamma^tr(X_t)dt\quad X_0 =
  x\label{eq:random-return}
\end{align}
Since $r$ is bounded, it follows that the discounted return is also bounded, and we express the space of returns by the interval $\returnspace=[V_{\min},V_{\max}]$. The \emph{value function} \citep{puterman2014markov} is the mapping $V^\pi:\xxX\to\returnspace$ defined pointwise by
\begin{equation*}
    V^\pi(x) \triangleq \expect [ \xxG{x} ] .
\end{equation*}
The distribution of $\xxG{x}$ is denoted by the probability measure $\returnmeasure^\pi(x)$ for each $x\in\xxX$:
\begin{align}
  \returnmeasure^\pi(x) &\triangleq \law{\xxG{x}}\label{eq:eta}
\end{align}
The mapping $\returnmeasure^\pi:\xxX\to\probset{\returnspace}$ is referred to as the \emph{return distribution function} \citep{bdr2022}. We equip $\returnspace$ with the Borel $\sigma$-algebra $\Borel(\returnspace)$ using the usual topology of the reals. 
We overload notation and write $\returnmeasure^\pi(x, A) = (\returnmeasure^\pi(x))(A)$.

The optimal control problem seeks a policy that maximizes the expected return. 
An \emph{optimal policy} $\pi^\star$ is one for which
\begin{align*}
    V^{\pi^\star}(x)\geq V^\pi(x) \qquad\forall\pi:\xxX\to\probset{\xxA}, \; x \in \xxX
\end{align*}
Because $\indexedabove{t}{X}$ is a Feller-Dynkin process, the value function is characterized by a partial differential equation (PDE) via the \emph{infinitesimal
  generator} $\xxL$ of the process.\footnote{Roughly, the infinitesimal generator of a Feller-Dynkin process with transition semigroup $(P_t)_{t\geq 0}$ satisfying $X_{t}\sim P_tX_0$ is the operator $\xxL$ satisfying $\xxL f = \lim_{t\downarrow 0}\mathbf{E}\frac{P_tf - f}{t}$ for each sufficiently smooth function $f$ on the state space. A formal definition is given in Appendix \ref{app:ctmp}.} This is established via the probabilistic
solutions to Kolmogorov backward PDEs \citep{kolmogoroff1931analytischen, le2016brownian}:
\begin{restatable}[Kolmogorov Backward Equation]{theorem}{kbe}\label{thm:kbe}
    Let $\indexedabove{t}{Y}:\R_+\to\mathcal{Y}$ be a Feller-Dynkin process in some space $\mathcal{Y}$, driven by an infinitesimal generator $\xxL$.
  Let $\mathcal{Y}^\circ\subset\mathcal{Y}$ be a measurable set with respect to the Borel $\sigma$-algebra on $\mathcal{Y}$, and let $S\in\R_+$ be the (random) exit time of $\indexedabove{t}{\mathcal{Y}}$ from $\mathcal{Y}^\circ$. It is assumed that $Y_0\in\mathcal{Y}^\circ$ and $\xxPSpaceMeasure(S<\infty) = 1$.
  For any
  measurable function $\phi$ that is absolutely continuous
  and differentiable almost everywhere, $u(t, y) =
  \mathbf{E}[\phi(Y_S)\mid Y_{t\land S} = y]$ solves
  \begin{equation}
    \label{eq:kbe}
    \partialderiv{u(t, y)}{t} = -\xxL u(t, y)
  \end{equation}
  with the terminal condition $u(t, y) = \phi(y)$ for all $y\notin \mathcal{Y}^\circ$.
\end{restatable}

The process $\indexedabove{t}{X}$ is called an \emph{\Ito\ diffusion} when
\begin{align}\label{eq:ito-diffusion}
    dX_t &= \xxf(X_t)dt + \xxs(X_t)dB_t
\end{align}
where $\xxf:\xxX\to\R^d, \xxs:\xxX\to\R^{d\times d}$ are the mean and diffusion of the stochastic dynamics of the agent controlled by the policy $\pi$, and $\indexedabove{t}{B}$ is a $\xxPSpaceMeasure$-Brownian motion. \footnote{An overview of Brownian motion is given in Appendix \ref{app:ctmp:brownian}.} Such processes have infinitesimal generators given by

\begin{equation}\label{eq:ito-diffusion:generator}
\begin{aligned}
  \xxL\psi(x) &= \langle\nabla_x\psi(x), \xxf(x)\rangle
  \ifthenelse{\equal{\conferencetemplate}{icml2022}}{\\&\qquad}{}
  + \frac{1}{2}\trace\left(\quadraticform{\xxs(x)}{\hessian{x}\psi(x)}\right)
\end{aligned}
\end{equation}
where $\hessian{x}$ is the Hessian operator with respect to $x$ and $\trace$ is the trace operator. Additionally, we assume that $\xxf,\xxs$ are continuous and differentiable almost everywhere, and that $\xxs(x)\succeq 0$ for each $x\in\xxX$.

Writing $u(t, x) = \ConditionExpect{\xxG{x}}{X_t=x} = V^\pi(x)$, we have
\begin{align*}
    \partialderiv{}{t}V^\pi(x) &= -\xxL V^\pi(x)
\end{align*}
Moreover, the Bellman equation \citep{bellman1957markovian} gives\marc{I think the $V$ in the equation below is missing a star. Then the sup will not make sense but ... gotta figure it out.}
\begin{align*}
    &V(X_t) = \sup_\pi \Expectation{\substack{X_{t+\Delta}\\\quad\sim P_\Delta,\pi}}{\int_0^\Delta\gamma^sr(X_{t + s})ds + \gamma^\Delta V^\pi(X_{t+\Delta})}\\
    &\partialderiv{}{t}V(X_t) = r(X_t) + \log\gamma V(X_t)
\end{align*}
Substituting into the Kolmogorov backward equation yields
\begin{align*}
    r(x) + \log\gamma V(x) + \sup_\pi\left\{\xxL V^\pi(x)\right\} = 0
\end{align*}
Expanding the expression for the generator in \eqref{eq:ito-diffusion:generator} we have
\begin{equation}\label{eq:hjb}\tag{HJB}
    \begin{aligned}
      &\sup_\pi\bigg\{r(x) + \langle\nabla V(x), \xxf(x)\rangle
      \ifconf{icml2022}{\\&\qquad\qquad}
      +\frac{1}{2}\trace\left(\quadraticform{\xxs(x)}{\hessian{}V(x)}\right)\bigg\}
      \ifconf{icml2022}{\\&\qquad\qquad\qquad}
      +\log\gamma V(x) = 0
    \end{aligned}
\end{equation}
which is the stochastic Hamilton-Jacobi-Bellman (HJB) equation \citep{fleming2006controlled}.

\subsection{Distributional reinforcement learning}

In \emph{distributional} RL, we aim to learn the probability
distribution $\returnmeasure^\pi(x)$ over $\xxG{x}$ as opposed to only its expectation.
A good approximation to the return distribution can be obtained by modelling particular statistics of the distribution, based on the notion of statistical functionals \citep{Rowland48495,bdr2022}.

\begin{definition}[Statistical functional, sketch]\label{def:statistical-functional}
A statistical functional maps each probability distribution to a real number.
A \emph{sketch} is a collection of statistical functionals, equivalently a mapping $\mathbf{s}:\probset[]{\R}\to\R^N$ that maps probability measures to ordered sets of real numbers (statistics).
\end{definition}
In this paper we will be interested in \emph{quantile functionals}, which effectively invert the CDF $\cdf{\nu}$ of a measure $\nu$. For a return distribution function $\returnmeasure$, let us write $\cdf{\returnmeasure}(x, z) = \returnmeasure(x, [V_{\min}, z])$. The quantile functionals $\quantile{\tau}$ are
\begin{equation*}
    \quantile{\tau}(\returnmeasure(x)) = \inf\left\{z\in\returnspace : \cdf{\returnmeasure}(x,z) = \tau\right\} \qquad \tau\in(0,1)
\end{equation*}
\begin{definition}[Imputation strategy]\label{def:imputation-strategy}
Define a set $\xxStatsSet{\Phi}\subset\R^N$ corresponding to a space of statistics. An \emph{imputation strategy} is a mapping $\Phi:\xxStatsSet{\Phi}\to\probset{\R}$. The set $\xxStatsSet{\Phi}$ is referred to as the set of \emph{admissible statistics} for $\Phi$.
\end{definition}
As with return distribution functions, for a vector $\statistics \in \xxStatsSet{\Phi}$ and $A \subseteq \R$ we write
\begin{equation*}
    \Phi(\statistics, A) = \Phi(\statistics)(A) .
\end{equation*}
We use imputation strategies to map statistical functional values back to distributions. In the sequel we consider the imputation strategy that maps a set of quantiles to what \citet{bdr2022} call a \emph{quantile distribution}.

\begin{definition}[Quantile Distribution]
  Let $\indexedint{k}{N}{y}$ be elements of a set $\mathcal{Y}$.
  The quantile distribution over $\mathcal{Y}$ with quantiles $\indexedint{k}{N}{y}$ is a
  probability measure $\nu$ given by
  \begin{align*}
    \nu(A) = \frac{1}{N}\sum_{k=1}^N\dirac{y_k}(A), \; A \in \Borel(\R) .
  \end{align*}%
\end{definition}%
Our aim will be to incorporate these two elements -- sketch and imputation strategy -- into a distributional HJB equation in order to produce a system of equations that can be approximated with standard numerical methods.

In our continuous-time formulations, we will analyze differential quantities of $\returnmeasure^\pi$ with respect to both the state space and the return space. As such, we will often find it more convenient to express return distributions $\returnmeasure^\pi(x)$ by their CDFs, which have a substantially simpler domain to differentiate over. We express these CDFs by $\cdf{\returnmeasure}:\xxX\times\returnspace\to[0,1]$, where $\cdf{\returnmeasure}(x, z) = \returnmeasure(x, [V_{\min},z])$.

\section{Distributional HJB Equations}\label{s:dhjb}
We will now shift our focus to formally representing the return
distribution function for an RL agent evolving continuously in time
with a fixed policy. In order to do so, it will be
necessary to impose some structural and regularity properties on the
dynamics of the environment and on the return distributions.


\begin{assumption}
  \label{ass:method:density}
  At every state $x\in\mathcal{X}$, the return distribution
  $\returnmeasure^\pi(x)$ is absolutely continuous with respect to the
  Lebesgue measure.
\end{assumption}

Although Assumption \ref{ass:method:density} can be violated in various MDPs,
particularly when dynamics are deterministic and the reward function is not
continuous, we note that such issues can easily be remedied in practice by adding low-variance white noise to the rewards, for example.

\begin{assumption}
  \label{ass:method:c2}
  The mapping $(x, z)\mapsto\cdf{\returnmeasure^\pi}(x, z)$ is twice differentiable over
  $\mathcal{X}\times\mathcal{R}$ almost everywhere, and its second
  partial derivatives are continuous almost everywhere.
\end{assumption}

All omitted proofs in the sequel will be provided in Appendix \ref{app:proofs}.

\subsection{Stochastic Return Processes}\label{s:dhjb:truncated-returns}
We would like to understand how estimates of the random return should
evolve over time, using the machinery of stochastic calculus \citep{le2016brownian}.
However, a function mapping states to (random) returns cannot be progressively 
measurable
(see Appendix \ref{app:stochastic}), as it requires knowledge of an entire 
trajectory.
Our solution is to introduce an intermediate stochastic process as a
``gateway'' to the random return.

\begin{definition}[The Truncated Return Process]\label{def:truncated-return}
  The \emph{truncated return process} is a stochastic process
  $(\xxJointFD_t)_{t\geq 0}\in\R_+\times\mathcal{X}\times\returnspace$ given by
  \begin{equation*}
    \begin{aligned}
      \xxJointFD_t = (t, X_t, \xxTruncRet_t)\qquad\xxTruncRet_t = \int_0^t\gamma^sr(X_s)ds
    \end{aligned}
  \end{equation*}
\end{definition}
  The values $\xxTruncRet_t$ are simply the discounted rewards
  accumulated up to time $t$, and $\xxTruncRet_0 = 0$.
  
\begin{proposition}\label{pro:markov}
  The truncated return process is a Markov process
  w.r.t. \hyperref[def:canonical-filtration]{the canonical filtration}.
\end{proposition}

The (discounted) random return can be expressed in terms of the
truncated return process. If the process $\indexedabove{t}{X}$ halts at the random exit time $T$, then $\xxTruncRet_T$ \emph{is} the return:
\begin{equation}\label{eq:random-return-truncated-return}
    \xxTruncRet_T\eqlaw \xxG{x}\qquad X_0 = x,
\end{equation}
where $\eqlaw$ denotes equality in distribution.
It will be convenient to encapsulate this identity in a time-homogeneous manner, since we would like to evaluate return distributions at each state as opposed to only the initial state $X_0$. This is captured by the \emph{conditional backward return process}.

\marc{If time, change $z$ here to be a more evocative variable (could be as simple as back-arrow-z.).}
\begin{definition}[Conditional Backward Return Process]\label{def:conditional-backward-return}
  Let $z\in\returnspace$ be a desired target return. The \emph{conditional backward return
  process} $\indexedabove{t}{\cbrprocess(z)}:\R_+\to\returnspace$ is given by
  \begin{align*}
    \cbrprocess(z)_t &= \gamma^{-t}(z - \xxTruncRet_t)
  \end{align*}
  Likewise, we define the joint process $\indexedabove{t}{\xxJointFDCB(z)}$ where $\xxJointFDCB(z)_t = (X_t, \cbrprocess(z)_t)$.
\end{definition}
Unlike the truncated return process which accumulates rewards ``forward in
time", the conditional backward return process conditions on a given return
$z$ and describes the residual discounted rewards needed to attain a return of
$z$. 

\subsection{A Characterization of the Return Distributions}\label{s:dhjb:characterization}

At this point, let us remark on the joint state-return process $(\xxJointFDCB(z)_t)_{t \ge 0}$. Because $X_t$ is $d$-dimensional and the return is bounded in $[V_{\min},V_{\max}]$, the joint process is effectively $(d+1)$-dimensional. Our goal is thus to derive, using the Kolmogorov backward equation, the PDE that characterizes the evolution of this joint process. The solution of this PDE is then the desired continuous-time return distribution function.
\begin{restatable}{lemma}{cbrkbe}\label{lem:cbr-kbe}
Let $z \in \returnspace$ be a desired return, and suppose that $\indexedabove{t}{\xxJointFDCB(z)}$ is a Feller-Dynkin process with infinitesimal generator $\xxL_J$. Then at each $x\in\mathcal{O}$ and $z'\in\returnspace$, $\returnmeasure^\pi$ satisfies
\begin{align}\label{eq:cbr-kbe}
    \xxL_J\cdf{\returnmeasure^\pi}(x, z') = 0
\end{align}%
\end{restatable}%
We are now ready to introduce the characterization of the return
distribution function in continuous time.
In the remainder of the text, the
notation $\proj{k}$ will be used to denote the coordinate projection
operators, where $\proj{k}(a_1, a_2,\dots,a_k,\dots) = a_k$.

\begin{restatable}[Distributional HJB Equation for Policy Evaluation]{theorem}{dhjb}
  \label{thm:dhjb}
  Denote by $\xxL_X$ the infinitesimal generator of the process $\indexedabove{t}{X} = \indexedabove{t}{\proj{1}\xxJointFDCB(z)}$.
  Moreover, suppose Assumptions \ref{ass:method:density} and
  \ref{ass:method:c2} hold.
  Then $\cdf{\returnmeasure^\pi}$ satisfies
  \begin{equation}
    \label{eq:dhjb}
    \small
    \begin{aligned}
      (\mathscr{L}_X\cdf{\returnmeasure^\pi}(\cdot, z))(x) -
      (r(x) + z\log\gamma)\partialderiv{}{z}\cdf{\returnmeasure^\pi}(x, z)
      = 0&
      \ifconfelse{icml2022}{\\}{&&}
      \xxPSpaceMeasure-\text{almost surely}&
    \end{aligned}
  \end{equation}
\end{restatable}
\marc{Is the derivative well-defined? Is this because of the assumptions?}
\harley{See proof}
Theorem \ref{thm:dhjb} admits a useful corollary when the agent
evolves according to an \Ito\ diffusion.

\begin{corollary}[Policy Evaluation of \Ito\ Diffusions]\label{cor:dhjb}
  In the setting of Theorem \ref{thm:dhjb}, if the state process
  $\indexedabove{t}{X}$ is governed by the \Ito\ diffusion of \eqref{eq:ito-diffusion}, the return distribution function $\returnmeasure^\pi$ satisfies for each $x \in \xxX$ and $z\in\returnspace$,
  \begin{equation}
    \label{eq:dhjb:ito}
    \small
    \begin{aligned}
      0 &= \langle\nabla_x\cdf{\returnmeasure^\pi}(x, z), \xxf(x)\rangle
      - \left(r(x) + z\log\gamma\right)\partialderiv{}{z}\cdf{\returnmeasure^\pi}(x, z)
      \ifconf{icml2022}{\\&\qquad}
      +\frac{1}{2}\trace\left(\quadraticform{\pmb{\sigma}_\pi(x)}{\hessian{x}\cdf{\returnmeasure^\pi}(x,
          z)}\right)
    \end{aligned}%
  \end{equation}%
\end{corollary}
\begin{proof}
  This result follows directly from Theorem \ref{thm:dhjb}, since the
  \hyperref[def:fd]{infinitesimal generator} $\mathscr{L}_X$ of an \Ito\
  diffusion is given by \eqref{eq:ito-diffusion:generator}.
\end{proof}
Note that the term $\partialderiv{}{z}\cdf{\returnmeasure^\pi}(x,\cdot)$ is the density of the return distribution at $x$. When the policy and environment dynamics are deterministic, we can relate Equation \eqref{eq:dhjb:ito} to \eqref{eq:hjb} by interpreting the derivatives using what is called the theory of distributions (an unfortunately-named class of objects that are usually not probability distributions; see Appendix \ref{app:distributions}), and setting $\pi = \pi^*$.


\subsection{Finitely-Parametrized Return Distributions}\label{s:dhjb:representation}
We now turn our attention to approximating the infinite-dimensional CDF $F_{\eta^\pi}(x, \cdot)$. Specifically, we consider what happens to Corollary \ref{cor:dhjb} when return distributions are represented by a \emph{statistics function} $\statistics{} : \xxX \to \xxStatsSet{\Phi}$, the statistical functional analogue of a value function $V : \xxX \to \R$. This statistics function corresponds to the values of $N$ statistical functionals, which can be transformed into a probability distribution by means of the imputation strategy $\Phi$. Consequently, we make the approximation
\begin{equation*}
    \returnmeasure^\pi(x) \approx \Phi(\statistics(x)) .
\end{equation*}
With this approximation, each return distribution can be represented in memory. The approximate distributional policy evaluation problem is then to determine a statistics function $\statistics$ that satisfies the It\^o Diffusion HJB.
In order to derive a robust characterization of the return
distribution function in the proposed manner, we will
require a mild regularity condition on the imputation strategy.

\begin{definition}[Statistical Smoothness]\label{def:statistical-smoothness}
  An imputation strategy $\Phi:\xxStatsSet{\Phi}\to\probset[p]{\returnspace}$ is said to be
  \emph{statistically smooth} if
  $\Phi(s)$ is a
  \hyperref[def:tempered-distribution]{tempered distribution} (see Appendix
  \ref{app:distributions}) for each $s\in\xxStatsSet{\Phi}$. Likewise, a return distribution function $\returnmeasure$ is said to be statistically smooth if $\cdf{\returnmeasure}(x, \cdot)$ is a tempered distribution for each $x\in\xxX$ and $\cdf{\returnmeasure}(\cdot, z)$ is twice continuously differentiable almost everywhere for each $z\in\returnspace$.
\end{definition}
\marc{One sentence explaining what this means. Also we use $\probset[p]{\cdot}$ above but $\probset[]{\cdot}$ below. Is $p$ used at all in this paper?}
\begin{definition}[Spatial Diffusivity]\label{def:diffusivity:spatial}
  Let $\Phi:\xxStatsSet{\Phi}\to\probset{\returnspace}$ be a statistically
  smooth imputation
  strategy, let $\statistics(x)$ be a statistics function, and suppose that $\indexedabove{t}{X}$ is governed by the \Ito\ diffusion of \eqref{eq:ito-diffusion}. 
  The \emph{spatial diffusivity} of the random return under the imputation
  strategy $\Phi$ is defined as the mapping
  $\statediffuse:\mathcal{X}\times\mathcal{R}\to\mathbf{R}^{d\times d}$ given by
  \begin{align*}
    \statediffuse(x,z) &=
    \sum_{k=1}^N\partialderiv{}{\proj{k}\statistics(x)}\Phi(\statistics(x), [V_{\min}, z])\hessian{x}\proj{k}\statistics(x)
  \end{align*}
  where $\hessian{x}$ is the Hessian operator with respect to $x$.
\end{definition}
Spatial diffusivity relates the stochasticity
of the approximate return distribution to the stochasticity of the state process. We will also identify a similar term relating the stochasticity of the return to the variability of the statistics as a result of the  stochasticity in the state process.\marc{Harley, this sentence is not clear. I tried to fix it a bit but I'm not sure what it's trying to say.}

\begin{definition}[Statistical Diffusivity]\label{def:diffusivity:statistical}
  Let $\Phi:\xxStatsSet{\Phi}\to\probset{\returnspace}$ be a statistically
  smooth imputation strategy and
  $\indexedabove{t}{X}:\R_+\to\mathcal{X}\subset\mathbf{R}^d$ the \Ito\ diffusion
  \eqref{eq:ito-diffusion}.
  The \emph{statistical
  diffusivity} of the random return under the imputation strategy $\Phi$ is
  defined as the mapping
  $\statsdiffuse:\mathcal{X}\times\returnspace\to\mathbf{R}^{d\times d}$ given
  by
  \begin{align*}
    \statsdiffuse(x,z) &=
    \quadraticform{\jacobian_x\statistics{}(x)}{\left(\hessian{\statistics(x)}\Phi(\statistics(x), [V_{\min}, z])\right)}
  \end{align*}
\end{definition}
We can now characterize the return distribution function as a PDE with respect to the
statistics function. Notably, this generalizes to all return distribution 
parameterizations that can be expressed by statistical functionals and imputation
strategies, such as those employed by categorical
\citep{Bellemare2017ADP}, quantile \citep{Dabney2018DistributionalRL},
and expectile \citep{Rowland48495} TD-learning.
\begin{restatable}[The Statistical HJB
Loss for Policy Evaluation]{theorem}{shjb}\label{thm:shjb}
Let $\Phi$ be a statistically smooth imputation strategy with a corresponding set of admissible statistics $\xxStatsSet{\Phi}$, and let $\statistics:\xxX\to\xxStatsSet{\Phi}$ be a statistics function. We define the mapping $\Psi(\statistics(x), z) = \Phi(\statistics(x), [V_{\min}, z])$. The \emph{Statistical HJB Loss} $\mathcal{L}_S$ is defined as
\begin{equation}\label{eq:shjb:loss}
\small
\begin{aligned}
      \ifconf{icml2022}{&}\mathcal{L}_S(\statistics, \Psi) =
      \ifconf{icml2022}{\\&\quad}
      \bigg[\measurement{\nabla_{\statistics(x)}\Psi(\statistics(x),
        z)}{\statistics{}_x(x)}{\xxf(x)}
        \ifconf{icml2022}{\\&\quad\ }
        -(r(x) + \log\gamma
      z)\partialderiv{}{z}\Psi(\statistics(x), z)\ifconf{neurips2022}{&}\\
      \ifconf{icml2022}{&\quad\ }
      +\frac{1}{2}\trace\left(\quadraticform{\pmb{\sigma}_\pi(x)}{\left(\statediffuse(x,z) +
        \statsdiffuse(x,z)\right)}\right)\bigg]^2\ifconf{neurips2022}{&}
\end{aligned}
\end{equation}
where $\statistics{}_x\triangleq\jacobian_x\statistics$. Let the assumptions of Corollary
\ref{cor:dhjb} hold. Then if $\cdf{\returnmeasure}$ satisfies \eqref{eq:dhjb:ito} and $\cdf{\returnmeasure(x)}=\Phi(\statistics(x))$ for each $x\in\xxX$, we have
  \begin{equation}
    \label{eq:shjb}
    \mathcal{L}_S(\statistics, \Phi) = 0
  \end{equation}
\end{restatable}
We define a statistical HJB \emph{loss}, as opposed to a PDE, since by restricting the return distribution function to a class that can be imputed by a given set of statistical functionals, \eqref{eq:dhjb} will not generally have a solution. However, analysis of \eqref{eq:shjb:loss} can reveal a lower bound on the approximation error, which can be useful when designing DRL algorithms in practice. Corollary \ref{cor:shjb} will demonstrate this.

While \eqref{eq:shjb} looks daunting, for certain imputation
strategies it can be simplified drastically. Imputation strategies that construct quantile distributions are particularly well-behaved in this regard, however, they necessitate a weakened interpretation of differentiability in order to make sense of the statistical HJB equation. Recall that quantile distributions are finite convex combinations of Dirac measures, so their CDFs are finite convex combinations of Heaviside functions. While these functions are in fact differentiable almost everywhere, their derivatives are zero, so all information about the distribution is lost under differentiation. When the return distribution function is statistically smooth,
however, we can reason about solutions to distributional HJB equations in \emph{the distributional sense}. 
\marc{``Under our assumption'' -- not an assumption; it's a definition.}
\harley{Unsure if this should go in Setting. I think it is less likely to go unnoticed here, not sure how important that is.}
For the purpose of the following results, $\psi'$ is said to be a \emph{distributional derivative} of the tempered distribution $\psi:\R\to\R$ if for every smooth and rapidly-decaying function $\rho:\R\to\R$, we have
\begin{align*}
    \int_{\R}\rho(z)\psi'(z)dz = -\int_{\R}\rho'(z)\psi(z)dz
\end{align*}
Intuitively, a distributional solution to a differential equation is a mapping that satisfies the equation upon convolution with every ``reasonable" smoothing kernel. This concept is discussed with more rigor in Appendix \ref{app:distributions}.

\marc{Because this is specifically the \emph{quantile} HJB, it would be good to use a specific notation for said functional (rather than the generic $\hat s$).}
\begin{restatable}[The Quantile HJB Equation for Policy Evaluation]{corollary}{shjbcor}
  \label{cor:shjb}
  Let $\returnmeasure^\pi$ be statistically smooth, and let $\mathbf{s}$ be the sketch that maps $\returnmeasure(x)$ to a quantile distribution for each $x\in\xxX$.

  If $\statistics(x) = \mathbf{s}(\returnmeasure^\pi(x))$ for each $x\in\xxX$ and
  $\cdf{\returnmeasure^\pi}$ is a distributional solution to \eqref{eq:dhjb:ito}, then sketch $\statistics$ of the statistical functionals $\indexedint{k}{N}{s}$
  is a distributional solution to the following system of PDEs,
  \begin{equation}
    \label{eq:shjb:dirac}
    \ifconf{icml2022}{\small}
    \begin{aligned}
        \begin{cases}
              \langle\nabla_{x}\proj{k}\statistics(x),\xxf(x)\rangle + r(x) + \log\gamma\proj{k}\statistics(x)\\
              \qquad\qquad + \frac{1}{2} \trace\left(\quadraticform{\xxs(x)}{\hessian{x}\proj{k}\statistics(x)}\right) = 0\\
              \proj{k}\statistics(x) = s_k(\returnmeasure(x))\\
              \hspace{0.75cm} k = 1,\dots,N
        \end{cases}
    \end{aligned}
  \end{equation}
\end{restatable}

Remarkably, this shows that the statistical diffusivity present in \eqref{eq:shjb}
vanishes under the quantile imputation strategy. The significance of this corollary
is twofold. Firstly, it demonstrates that under the quantile representation, 
distributional dynamic programming reduces to solving a system of HJB equations,
so existing HJB solving methods can be leveraged (such as that of 
\citet{munos1997reinforcement}) for continuous-time distributional RL. Moreover,
comparing \eqref{eq:shjb:dirac} and \eqref{eq:shjb:loss}, it is clear that
such a reduction \emph{is not possible} in general -- in particular, to adapt 
categorical or expectile TD-learning algorithms to the continuous-time setting, one
must take extra care to account for the spatial and statistical diffusivity due to
the corresponding imputation strategies.
%
\section{A Reinforcement Learning Algorithm}\label{s:algorithms}
We propose a model-based DRL algorithm for jointly learning the return distribution 
function and optimizing the policy. Our algorithm is akin to the Quantile Regression 
TD-Learning (QTD) algorithm \citep{Dabney2018DistributionalRL} with two important 
differences: \emph{(a)} we update return distributions according to the Quantile HJB 
equation \eqref{eq:shjb:dirac} as opposed to the distributional Bellman equation 
\citep{Bellemare2017ADP}, and \emph{(b)} we employ a differential updating scheme that 
converges in the limit of continuous time updates as opposed to a simple gradient descent.

In order to adapt a TD-learning algorithm like QTD to the continuous-time setting, we must 
note that TD updates may occur at arbitrarily high frequencies -- as such, return 
distributions can evolve continuously in time. We must ensure that our update scheme is 
well-defined in this limit.
To do so, we model a gradient \emph{flow} as opposed to a sequence of gradient updates. \citet{Jordan02thevariational} presents the \emph{JKO scheme} to accomplish this in the $2$-Wasserstein space when the loss functional has the form $\mathscr{F}(\returnmeasure)=\int_{\returnspace} Ud\returnmeasure - \frac{1}{\beta}\xxEntropy(\returnmeasure)$, where $\xxEntropy$ denotes entropy. We must derive the function $U$ such that the loss is minimized at the return distribution function. This is precisely what is done by \citep{martin2020stochastically} to minimize the distributional Bellman error.

In order to adapt the JKO scheme of \citet{martin2020stochastically}, we replace the distributional Bellman error with a signal that we refer to as the \emph{kinetic energy of returns}.
We use a quantile imputation strategy $\Phi$ to approximate return distribution functions, so return distributions can be interpreted as finite sets of $N$ return ``particles" each having equal mass. For a set of particles distributed by $\returnmeasure(x)\in\probset{\returnspace}$ for $x\in\xxX$, let $\Psi(x, z) = \Phi(\statistics(x), z) = \cdf{\returnmeasure}(x, z)$.
Denoting by $\xxL$ the infinitesimal generator of the conditional backward return process, we define the kinetic energy according to
\begin{align}
    U(z) &= \frac{1}{2}\left(\xxL\Psi(x, z)\right)^2\label{eq:kinetic-energy}
\end{align}
This results in the loss $\xxLoss_\beta:\probset{\returnspace}\to\R$ given by
\begin{equation}
  \label{eq:loss:beta}
  \xxLoss_\beta(\returnmeasure(x)) =
  \int_{\returnspace}\frac{1}{2}\left(\xxL\Psi(x, z)\right)^2\returnmeasure(x, dz)
  - \frac{1}{\beta}\xxEntropy(\returnmeasure)
\end{equation}
Under the quantile distribution representation, we see in \eqref{eq:shjb:dirac} that $\xxL\Psi(x,\cdot)$ is affine. Therefore, the kinetic energy is
convex, and then it is a well established result that $\xxLoss_\beta$
is convex \citep{ambrosio2008gradient}. Therefore, $\xxLoss_\beta$ has
a unique (global) minimum. We consider the gradient flow of
\eqref{eq:loss:beta}:
\begin{equation}
  \label{eq:gradient-flow}
  \returnmeasure_s(x) =
  -\nabla\xxLoss_\beta(\returnmeasure_s(x,\cdot))
\end{equation}
where $s$ is a continuous time parameter.\footnote{This is not
  necessarily equivalent to the time parameter in the MDP.}
Remarkably, \citet{Jordan02thevariational} shows that
\eqref{eq:gradient-flow} in the $2$-Wasserstein space is equivalent to
the \emph{Fokker-Planck equation}, which is a well-known PDE in
various scientific disciplines. As a result of this, it is well known
that \eqref{eq:loss:beta} is minimized when the density $\varrho$ of $\returnmeasure$ satisfies
$\returnmeasure\propto\exp(-\beta U)$.

Furthermore, \citet{martin2020stochastically} shows that as $\beta\to\infty$, the minimizer of $\xxLoss_\beta$ coincides with $U\equiv 0$. With $U$ given by \eqref{eq:kinetic-energy}, this occurs when $\returnmeasure(x)$ satisfies the Kolmogorov backward equation for $\xxL$. By Theorem \ref{thm:shjb}, we see that the loss is minimized by the return distribution function.

To construct a reinforcement learning algorithm, we must discretize the gradient flow \eqref{eq:gradient-flow} in time. The JKO scheme for \eqref{eq:gradient-flow} consists of computing the sequence of iterates $\sequence{k}{\xxMMSIterate}$ given by
\begin{equation}
  \label{eq:minimizing-movements}
  \xxMMSIterate_{k+1}\in\arg\min_{\returnmeasure}\bigg[2\tau\int_{\returnspace}Ud\returnmeasure(x)
  + W^{\beta}_2(\returnmeasure, \xxMMSIterate_{k})\bigg]
\end{equation}
where $W^\beta_2$ is the entropically-regularized $2$-Wasserstein distance \citep{cuturi2013sinkhorn} with inverse temperature $\beta$. Computation of this distance is tractable for quantile distributions via the \emph{Sinkhorn algorithm} \citep{cuturi2013sinkhorn,martin2020stochastically}. Remarkably, the Sinkhorn algorithm is differentiable \citep{peyre2019computational}, which allows us to incorporate it with gradient-based optimization schemes.

Under a continuous time interpolation of $\sequence{k}{\xxMMSIterate}$ given by \citet{Jordan02thevariational}, the interpolated curve converges to \eqref{eq:cauchy} as $\tau\to 0$ in \eqref{eq:minimizing-movements}.

\subsection{Control}
In continuous time, individual actions have negligible effects on the return, and consequently
the action-value function cannot be used to infer optimal actions \citep{baird1993advantage}.
This concept is formalized by \citet{bellemare2016increasing} and \citet{tallec2019making}.
To account for this, \emph{advantage-updating} \citep{baird1993advantage} and similar schemes 
\citep{bellemare2016increasing} introduce alternative notions of action values that are
meaningful in the continuous-time limit. However, to the best of our knowledge, such concepts
have not been studied in a distributional framework. Since the theory that was presented in 
this paper is concerned only with policy evaluation, such developments are out of scope, but
are certainly interesting avenues for future work. 

In order to perform simulations, we must discretize time. When time is discretized, individual 
actions are no longer negligible\footnote{Note that, while individual actions may have discernible influence on the return in this setting, their influence is still small. Consequently, due to noise in the training process, convergence to an optimal policy can be quite slow as reported by \citet{baird1993advantage}, so it would still behoove us to study a distributional analogue to advantage updating even in the time-discretized setting.}, so state-action pairs will not be completely invariant to the action. 

We associate $|\xxA|$ return distributions to each state (one per action), and henceforth we use the notation $\returnmeasure^\pi(x, a)$ to denote the return distribution associated to the policy $\pi$ corresponding to the state-action pair $(x,a)$. Likewise, statistics functions are indexed by actions, so we now write $\Phi(\statistics(x, a), z) = \returnmeasure(x, a)$.

In order to infer an optimal policy given a return distribution function, we must impose an ordering among return distributions. We simply order return distributions by their expected values, akin to many common DRL algorithms \citep{Bellemare2017ADP, Dabney2018DistributionalRL}. Subsequently, we deem a policy $\pi^\star$ optimal if, for every state-action pair $(x, a)$, $\returnmeasure^{\pi^\star}(x, a)$ has greater expectation than $\returnmeasure^\pi(x, a)$ for any other policy $\pi$.

\subsection{Approximating Solutions to the DHJB Equation}
\newcommand{\fddrift}{\widehat{\mu}}
\newcommand{\fdsig}{\widehat{\pmb{\Sigma}}}
In order to maintain estimates of the return distribution function at each state-action pair, we must discretize the state space to a finite collection of points. Consequently, we must derive approximations of the differential terms in \eqref{eq:shjb:dirac}. We will write $f^\pm(x) = \max(\pm f(x), 0)$, and we will approximate the drift and variance of the dynamics by $\fddrift:\xxX_\epsilon\times\xxA\to\xxX\approx\xxf$ and $\fdsig:\xxX_\epsilon\times\xxA\to\R^{d\times d}\approx\xxs\xxs^\top$.

Fortunately, \eqref{eq:shjb:dirac} has a very special form: it is simply a system of HJB equations. Due to the prevalence of HJB equations in control and continuous-time RL research, there are myriad established methods for solving them. We will make use of the finite-differences scheme that was introduced by \citet{munos1997reinforcement}, which approximates solutions to HJB equations driven by \Ito\ diffusions.

For some $\epsilon>0$, we discretize $\xxX$ to a lattice $\xxX_\epsilon = \{\sum_{n=1}^di_n\epsilon\vec{e}_n : i_n\in\Z\}\cap\xxX$ where $\indexedint{i}{d}{\vec{e}}$ is the standard basis of $\R^d$.
The \emph{neighbors} of each state $\xi\in\xxX_\epsilon$ are points adjacent to $\xi$ in the lattice. The mapping $\fdneighbor:\xxX_\epsilon\to 2^{\xxX_\epsilon}$ maps each state to the set of its neighbors, given as follows,
\begin{align*}
    \fdneighbor(\xi) = \{\xi'\in\xxX_\epsilon\setminus\{\xi\} :\ &\xi' = \xi + a\epsilon\vec{e}_i + b\epsilon\vec{e}_j,
    \ifconfelse{icml2022}{\\&}{\ }
    i,j\in [d],\ i\neq j,
    \ifconfelse{icml2022}{\\&}{\ }
    a,b\in\{0,\pm 1\}\}
\end{align*}
Since the state space is divided into a finite collection of ``cells", all states within a given cell are indistinguishable from one another in our approximation. As such, at any given cell, even if the dynamics are deterministic, the agent's future states are randomly distributed. This phenomenon is depicted in Figure \ref{fig:lattice}.
\begin{figure}[h]
\centering
\includegraphics[scale=0.8]{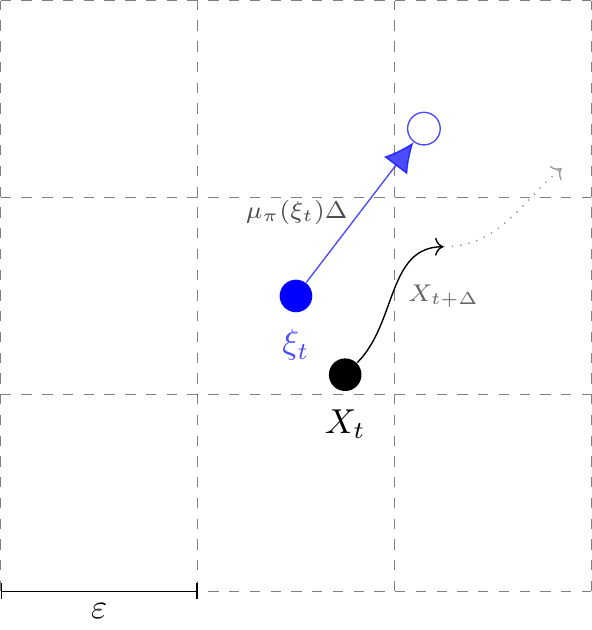}
\caption{Finite-differences approximate trajectory (blue) relative to real trajectory (black).}\label{fig:lattice}
\end{figure}
Suppose the agent has velocity $\vec{v}$ with speed $v$. For the timestep $\Delta$ for which $\Delta v = 1\cdot\epsilon$, the components of $\vec{v}$ can be interpreted as probabilities $p\epsilon$ as shown in Figure \ref{fig:lattice}. When the dynamics are driven by an \Ito\ diffusion, \citet{munos1997reinforcement} shows that this timestep is given by
\begin{align*}
\Delta_{\xi,a} &= \frac{\epsilon^{2}}{\epsilon\norm[1]{\fddrift(\xi,a)} + \trace(\fdsig(\xi,a)) - \frac{1}{2}\sum_{j\neq i}|\fdsig(\xi,a)_{ij}|}\\
\end{align*}
Subsequently, the transition probabilities are given by
\begin{align*}
\ifconf{icml2022}{&}p(\xi, a, \xi\pm\epsilon\vec{e}_i)
\ifconfelse{icml2022}{=\\}{&=}
\ifconf{icml2022}{&\quad}\frac{\Delta_{\xi,a}}{2\epsilon^2}\left[2|(\proj{i}\fddrift)^\pm(\xi,a)| + \fdsig(\xi,a)_{ii} - \sum_{j\neq i}|\fdsig(\xi,a)_{ij}|\right]\\
\ifconf{icml2022}{&}p(\xi, a, \xi+\epsilon(\vec{e}_i\pm\vec{e}_j)) 
\ifconf{neurips2022}{&}=
\frac{\Delta_{\xi,a}}{2\epsilon^2}\fdsig^{\pm}(\xi,a)_{ij}\quad i\neq j\\
\ifconf{icml2022}{&}p(\xi, a, \xi-\epsilon(\vec{e}_i\pm\vec{e}_j))
\ifconf{neurips2022}{&}=
\frac{\Delta_{\xi,a}}{2\epsilon^2}\fdsig^{\pm}(\xi, a)_{ij}\quad i\neq j\\
\ifconf{icml2022}{&}p(\xi, a, \xi') \ifconf{neurips2022}{&}= 0\quad\text{otherwise}
\end{align*}
We define the finite differences distributional Bellman operator by $\mathcal{T}_{\Delta}$ where
\newcommand{\boot}[3][\Delta_{\xi,a}]{\mathsf{b}_{#2, #3}^{#1}}
\newcommand{\bootr}[2][r]{\boot{#1}{#2}}
\newcommand{\bootry}[1][\gamma]{\bootr{#1}}
\begin{equation}\label{eq:fd}%
\small
    \begin{aligned}
        &\pi^\star_\xi\leftarrow\underset{a'\in\xxA}{\argmax}\Expect{\Phi(\xi, a')}\qquad\forall \xi\in\xxX_\epsilon\\
        &\bootry:\returnspace\to\returnspace:z\mapsto \Delta_{\xi,a}r + \gamma^{\Delta_{\xi,a}}z\\
        &\qquad\qquad\qquad\qquad i\neq j\in[d],\ a,b\in\{0,\pm 1\}\}\\
        &\mathcal{T}_\Delta\Phi(\xi, a) \leftarrow \sum_{\xi'\in\fdneighbor(\xi)}p(\xi, a, \xi')\left(\bootry\right)_\sharp\Phi(\xi', \pi^\star_{\xi'})\\
    \end{aligned}
\end{equation}%
where $\sharp$ denotes the pushforward operation, defined by $f_\sharp\mu = \mu\circ f^{-1}$ for a measure $\mu$ and a measurable function $f$.
Finally, the finite differences approximation of \eqref{eq:shjb:dirac} is the fixed point equation
\begin{equation}\label{eq:shjb:fd}
\mathcal{T}_{\Delta}\Phi(\xi, a) = \Phi(\xi, a)\qquad\xi\in\xxX_\epsilon,\ a\in\xxA
\end{equation}
We derive an algorithm based on these principles as an iterative method for solving 
\eqref{eq:shjb:fd}. Notably, with the quantile representation, our algorithm is tractable
relative to a HJB-solving oracle, such as the algorithm proposed by \citet{munos1997convergent}. 
The learning update is summarized in Algorithm 
\ref{alg:fd}, which can be applied in both online and offline settings. When the dynamics 
$\xxf,\xxs$ are unknown, which is usually the case in reinforcement learning, they can be
estimated by the sample mean and sample covariance of observed transitions, respectively
\citep{munos1997reinforcement}. The algorithm has access to a mapping
$\mathsf{Enc}:\xxX\to\xxX_\epsilon$ which maps states to their closest point in the 
lattice $\xxX_\epsilon$. For the purpose of exploration, we simply employ a 
$\epsilon$-greedy policy.

\begin{algorithm}[h]
  \begin{algorithmic}
    \REQUIRE{WGF time parameter $\tau$}
    \REQUIRE{Learning rate $\alpha$}
    \REQUIRE{State transition $(x, a, r, x')$}
    \REQUIRE{Duration of transition $\delta$}
    \STATE{$\xi\leftarrow\mathsf{Enc}(x)$}
    \\\COMMENT{Update model}
    \STATE{$\fddrift(\xi,a)\leftarrow(1-\alpha)\fddrift(\xi,a) + \alpha(x'-x)/\delta$}
    \STATE{$\sigma\leftarrow x' - x - \Delta\fddrift(\xi)$}
    \STATE{$\fdsig(\xi,a)\leftarrow(1-\alpha)\fdsig(\xi,a) + \alpha\Delta^{-1}\sigma\sigma^\top$}
    \\\COMMENT{Compute mixture of target quantiles}
    \FOR{$y\in\fdneighbor(\xi)$}
        \STATE{$(\mathbf{T}_{\Delta_{\xi,a}})_{y}\leftarrow\Delta_{\xi,a}r + \gamma^{\Delta_{\xi,a}}\statistics(y, \pi^\star_y)$}
        \STATE{$\mathbf{p}_y\leftarrow p(\xi, a, y)$}
    \ENDFOR
    \STATE{$\widehat{\returnmeasure}\leftarrow\frac{1}{N}\sum_{y\in\fdneighbor(\xi)}\mathbf{p}_y\sum_{k=1}^N\dirac{(\mathbf{T}_{\Delta_{\xi,a}})_{y,k}}$}
    \\\COMMENT{Update quantiles}
    \STATE{$\returnmeasure_0\leftarrow\frac{1}{N}\sum_{k=1}^N\dirac{\statistics(\xi, a)_k}$}
    \STATE{$\returnmeasure\leftarrow\underset{\nu\in\probset{\returnspace}}{\argmin}\ \left[2\tau\Expectation{Z'\sim\widehat{\returnmeasure},Z\sim\nu}{(Z-Z')^2} + W_2^\beta(\nu,\returnmeasure_0)\right]$}
    \STATE{$\statistics(\xi, a)\leftarrow\mathbf{s}(\returnmeasure)$}\COMMENT{Extract quantiles of return distribution}
  \end{algorithmic}
  \caption{Continuous-time distributional RL update}\label{alg:fd}
\end{algorithm}
\section{A Qualitative Demonstration}\label{s:experiments}
We simulate the performance of the FD-WGF
$Q$-learning algorithm on a simple task based on a continuous MDP suggested by
\citet{Munos2004ASO} as an example of an MDP whose value function does not satisfy the HJB equation in the usual sense. In
this environment, we control a particle on $\xxX = [0,1]$ with actions
among $\xxA = \{-1,1\}$. The dynamics of the particle are given by
$\dot{x}(t) = a(t)$.

Rewards are zero in the interior of $\xxX$, and are otherwise sampled
from $\mathcal{N}(2, 2)$ and $\mathcal{N}(1, 1)$ at states $1, 0$
respectively. The discount factor is $\gamma = 0.3$, and observations
occur at a frequency $\omega=1\text{kHz}$. We observe the performance
of FD-WGF $Q$-learning relative to the Quantile Regression TD-learning
algorithm (QTD) proposed by
\citet{Dabney2018DistributionalRL}. Figure
\ref{fig:birdseye} depicts an overview of the return distribution
functions learned by both algorithms.
\begin{figure}[h!]
  \centering
  \ifconfelse{icml2022}{%
  \newcommand{\dtscale}{0.226}
  \newcommand{\ctscale}{0.339}
  QTD\\
  \includegraphics[scale=\dtscale, trim={0 0 0 2cm},
  clip]{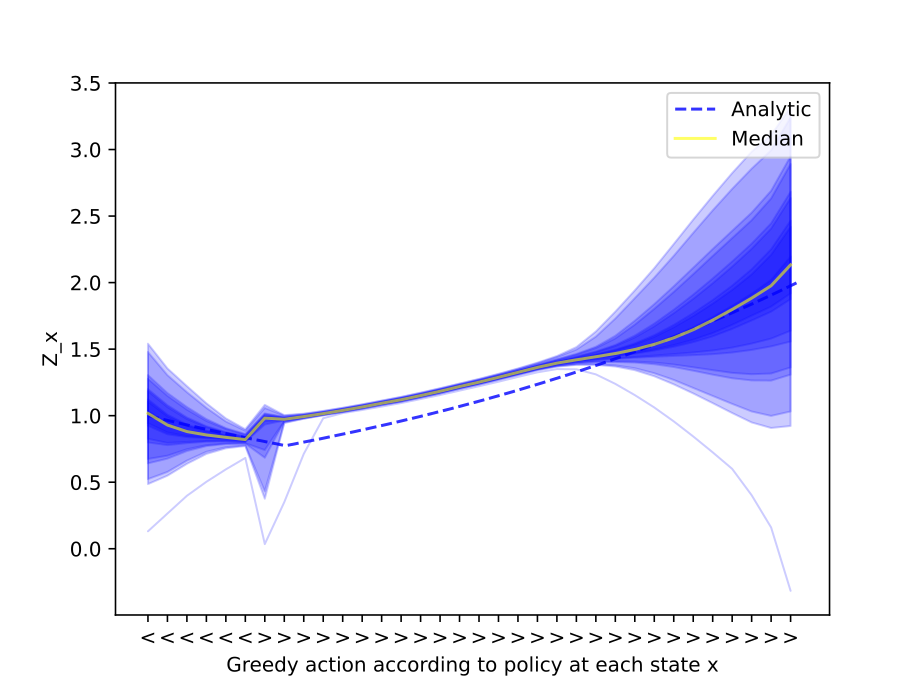}\\
  FD-WGF $Q$-learning\\
  \includegraphics[scale=\ctscale, trim={0 0 0 1.5cm}, clip]{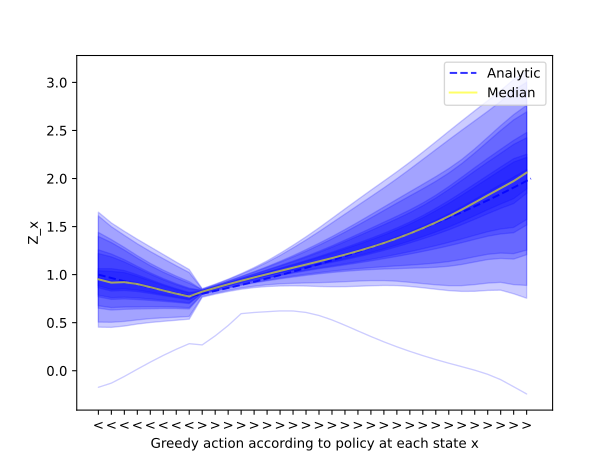}%
  }{%
  \newcommand{\dtscale}{0.226}
  \newcommand{\ctscale}{0.339}
  \begin{minipage}{0.49\textwidth}
  \centering
  QTD\\
  \includegraphics[scale=\dtscale, trim={0 0 0 2cm},
  clip]{results/dt51-munos-overview}\\
  \end{minipage}
  \begin{minipage}{0.49\textwidth}
  \centering
  FD-WGF $Q$-learning\\
  \includegraphics[scale=\ctscale, trim={0 0 0 1.5cm}, clip]{results/ct51-munos-overview-0}%
  \end{minipage}
  }
  \caption{Return distribution functions and policies learned by
    FD-WGF $Q$-learning and QTD}
  \label{fig:birdseye}
\end{figure}
In Figure \ref{fig:birdseye}, the darker blue regions represent larger
probability mass of the return distribution. The dashed blue line is
the analytical value function. We observe that our proposed algorithm
learns a good representation of the value function, whereas the QTD
algorithm tends to fail near the point of non-differentiability of the
value function. Consequently, we see that QTD overestimates the value
function over much of the state space. Figure \ref{fig:scan} shows the
return distributions learned by each algorithm near the boundaries of
the state space, where the return variance is greatest.

We observe that FD-WGF $Q$-learning represents the true return
distribution far more accurately near $\partial \xxX$, while both
algorithms tend to ``lose'' variance further away from the
boundaries. That said, especially when $x>0.5$, we see that QTD has
substantially more difficulty learning the variance of the return
distributions than FD-WGF $Q$-learning.
\section{Conclusion}\label{s:conclusion}
Our work demonstrates that extra care should be taken when
designing distributional RL algorithms for continuous-time
problems. Notably, we have shown that the approximation of return
distributions as empirical distributions is particularly well suited
to continuous-time problems, as these representations eliminate the
\emph{statistical diffusivity} of the return due to the stochasticity
of the system. Through our simulated experiments, we confirmed the
hypothesis that accounting for continuous time aids DRL algorithms to
preserve the return distribution entropy.

The algorithm presented in this work, as a finite-differences based scheme, becomes intractable as the dimension $d$ of the state space grows. However, we note that function approximation can be integrated without much difficulty to account for these cases. Since the loss function is differentiable, we can envision an algorithm similar to Algorithm \ref{alg:fd} with $\statistics$, $\xxf$, and $\xxs$ parameterized by neural networks, with the gradient and Hessian of $\statistics$ computed via automatic differentiation and parameters trained via gradient descent. This algorithm would be similar to \emph{Online WGF Fitted $Q$-iteration} \citep{martin2020stochastically}, which demonstrates promising results. Such extensions are left for future work.

\section*{Acknowledgements}
The authors thank Julie Alhosh, Matthieu Geist, and Mark Rowland for their detailed and constructive feedback on previous drafts, as well as the anonymous reviewers whose insights helped improve the paper.
This work was supported by the National Sciences and Engineering Research 
Council (NSERC) Canadian Robotics Network (NCRN), as well as the Canada CIFAR 
AI Chair program.
Harley Wiltzer's research was supported by a NSERC CGS-M grant.
\bibliography{sources}
\bibliographystyle{icml2022/icml2022}

\newpage
\appendix
\onecolumn
\section{Proofs}\label{app:proofs}
\subsection{Proofs of Results in \S\ref{s:dhjb}}\label{app:proofs:dhjb}
\begin{proof}[Proof of Proposition \ref{pro:markov}]
\marc{Consider defining the $C$ classes.}
  Let $\psi\in C(\mathcal{X}\times\mathcal{R};\mathbf{R})$ and $h>0$. As usual,
  we denote the canonical filtration by $(\mathcal{F}_t)_{t\geq 0}$.
  By the definition of the truncated return process,

  \begin{equation*}
    \small
    \begin{aligned}
      \ConditionExpect{\psi(J_{t+h})}{\mathcal{F}_t} &=
      \ConditionExpect{\psi(X_{t+h}, \overline{G}_{t+h})}{\mathcal{F}_t}\\
      &= \ConditionExpect{\psi\left(X_{t+h}, \overline{G}_t + \int_t^{t+h}\gamma^sr(X_s)ds\right)}{\mathcal{F}_t}\\
      &= \ConditionExpect{\psi\left(X_{t+h}, \overline{G}_t + \int_t^{t+h}\gamma^sr(X_s)ds\right)}{J_t}\\
    \end{aligned}
  \end{equation*}

  where the final step holds since the process $(X_t)_{t\geq 0}$ is
  assumed to be Markovian. Thus, we've shown that for any $\psi\in
  C(\mathcal{X}\times\mathcal{R};\mathbf{R})$, there exists a function
  $m:\mathcal{X}\times\mathcal{R}\to\mathbf{R}$ where

  \begin{equation*}
    \ConditionExpect{\psi(J_{t+h})}{\mathcal{F}_t} = m(X_t, \overline{G}_t)
  \end{equation*}

  Therefore, the process $(J_t)_{t\geq 0}$ is Markovian.
\end{proof}

\begin{restatable}{lemma}{dhjbfv}\label{lem:finite-variation}
  Let $\indexedabove{t}{\xxJointFD}=(X_t, \xxTruncRet_t)_{t\geq 0}$ be
  the \hyperref[def:truncated-return]{truncated return process} defined
  in Theorem \ref{thm:dhjb}.\marc{Def 5?}\harley{I reordered some things before I saw this, and I added the definition of FV process, is it clear now?} Then $\indexedabove{t}{\xxTruncRet}$ is
  a finite variation process.
\end{restatable}

In order to determine the infinitesimal generator of the truncated return process, it will be necessary to estimate its quadratic variation and the bracket $([X,\xxTruncRet]_t)_{t\geq 0}$. Establishing $\indexedabove{t}{\xxTruncRet}$ as a finite variation process will greatly simplify this estimate.

\begin{proof}[Proof of Lemma \ref{lem:finite-variation}]
  By definition, we have

  \begin{equation*}
    \overline{G}_t = \int_0^t\gamma^sr(X_s)ds
  \end{equation*}

  Consider the measurable space $(\mathbf{R}_+, \Sigma)$ where
  $\Sigma$ is the $\sigma$-algebra of Lebesgue-measurable subsets of
  the nonnegative reals, and let $\Lambda$ denote the Lebesgue
  measure. We will use $(\mathbf{R}_+,\Sigma)$ to measure \emph{time}. 
  By the Radon-Nikodym theorem, for each sample path $\omega\in\Omega$, the function
  $\mu_\omega:\Sigma\to\mathbf{R}$ shown below is a signed measure on this
  measurable space,

  \begin{equation*}
    \mu_\omega(A) = \int_A\gamma^{s\land T(\omega)}r(X_{s\land T(\omega)}(\omega))\Lambda(ds)\qquad A\in\Sigma
  \end{equation*}

  Then, for any $\omega\in\Omega$, the mapping $t\mapsto G_t(\omega) =
  \mu_\omega([0,t])$. This shows that each sample path is a function
  $a:t\mapsto\mu_\omega([0,t])$ for the measure $\mu_\omega$, so every sample path
  is a finite variation function by definition.\marc{Expand on this proof -- as is it's a bit of a sleigh of hand.}\harley{Is this proof less mysterious given the appendix on FV processes?}
\end{proof}

\begin{restatable}{lemma}{dhjbgen}\label{lem:generator}
  The truncated return process $\indexedabove{t}{\xxJointFD}$ as defined in
  Theorem \ref{thm:dhjb} is a \hyperref[def:fd]{Feller-Dynkin process}.
\end{restatable}

\begin{proof}
  Consider the \hyperref[def:filtration]{filtered probability space}
  $\mathsf{P} = (\Omega, \mathcal{F},
  \indexedabove{t}{\mathcal{F}}, \Pr)$ defined
  previously. Proposition
  \ref{pro:markov} shows that $\indexedabove{t}{J}$ is a Markov process. It remains to
  show that it is a Feller-Dynkin process. First, we must show that
  its transition semigroup maps $\indexedabove{t}{P}$ are
  endomorphisms on $C_0(\R_+\times\mathcal{X}\times\returnspace)$. Let
  $\psi\in C_0(\R_+\times\mathcal{X}\times\returnspace)$.

  Note that since $\indexedabove{t}{X}$ has continuous sample paths,\marc{Why?}
  $\indexedabove{t}{\overline{G}}$ has absolutely continuous sample
  paths since\marc{Do we need assumption on $r$?}

  \begin{equation*}
    \overline{G}_t(\omega) = \int_0^t\gamma^sr(X_s(\omega))ds \qquad
    \omega\in\Omega
  \end{equation*}

  so it is bounded by the integral of a bounded function. Therefore
  $P_\delta\psi$ can be expressed as

  \begin{align*}
    P_\delta\psi &= \int \psi\circ (t+\delta, X_{t+\delta}, \overline{G}_{t+\delta})d\Pr
  \end{align*}

  Since the sample paths $X_{t+\delta},\overline{G}_{t+\delta}$ are
  continuous, the integrand above is a continuous
  function. Additionally, since $\psi,\mathcal{X},\returnspace$ are
  all compactly supported, we see that $P_\delta\psi$ is as
  well. Therefore $P_\delta\psi\in C_0(\R_+\times\mathcal{X}\times\returnspace)$.

  It is easy to check that $P_0\psi=\identity$. This follows simply
  from the fact that $\indexedabove{t}{X}$ is a Feller-Dynkin process
  (so its semigroup has an identity) and
  $\indexedabove{t}{\overline{G}}$ is deterministic given
  $\indexedabove{t}{X}$. For the same reason, it follows that $P_tP_s=P_{t+s}$.

  It remains to show that $\norm[\infty]{P_\delta\psi - P_0\psi}\overset{\delta\downarrow
    0}{\longrightarrow} 0$. We have

  \begin{align*}
    \norm[\infty]{P_\delta\psi - P_0\psi} &= \norm[\infty]{P_\delta\psi - \psi}\\
    &= \norm[\infty]{\int_{\xxX\times\returnspace}\left(\psi\circ(t+\delta,X_{t+\delta},
      \overline{G}_{t+\delta}) - \psi(t, X_t, 
      \overline{G}_t)\right)d\xxPSpaceMeasure}\\
    &= \norm[\infty]{\int_{\xxX\times\returnspace}\psi\circ(t+\delta, X_{t+\delta},
      \overline{G}_{t+\delta})d\xxPSpaceMeasure - \psi(t, X_t, \overline{G}_t)}\\
  \end{align*}

  Since $\psi$ is supported on a compact finite-dimensional set\marc{Make sure this is true} and it
  is continuous, it follows that it is bounded. Therefore, it follows
  by the dominated convergence theorem that

  \begin{align*}
    \lim_{\delta\to 0}\int\psi\circ(t+\delta, X_{t+\delta},
    \overline{G}_{t+\delta})d\Pr &= \int\psi\circ\lim_{\delta\to
                                   0}(t+\delta,X_{t+\delta}, 
                                   \overline{G}_{t+\delta})d\Pr\\
                                 &= \int\psi(t, X_t, \overline{G}_t)d\Pr\\
    &= \psi(t, X_t, \overline{G}_t)
  \end{align*}
  This proves the claim.
\end{proof}

\begin{restatable}{lemma}{dhjbcorgenerator}\label{cor:generator}
  The truncated return process $\indexedabove{t}{J}$ defined in
  Theorem \ref{thm:dhjb} has an infinitesimal generator
  $\mathscr{L}:C_0(\R_+\times\mathcal{X}\times\returnspace)\to
  C_0(\R_+\times\mathcal{X}\times\returnspace)$
  given by
  \begin{equation}
    \label{eq:truncated-return:generator}
    \mathscr{L}\psi(t, x, \overline{g}) = (\mathscr{L}_X\psi(t, \cdot, \overline{g}))(x) +
    \gamma^tr(x)\partialderiv{}{\overline{g}}\psi(t, x, \overline{g}) +
    \partialderiv{}{t}\psi(t, x, \overline{g})
  \end{equation}
  where $\mathscr{L}_X$ is the infinitesimal generator of the process
  $\indexedabove{t}{\proj{2}J} = \indexedabove{t}{X}$.
\end{restatable}

\begin{proof}
  Since Lemma \ref{lem:generator} shows that $\indexedabove{t}{J}$ is
  a \hyperref[def:fd]{Feller-Dynkin process}, the existence of an
  infinitesimal generator driving this process is guaranteed.
  Let $\psi\in C^2_0(\R_+\times\mathcal{X}\times\returnspace)$ and denote $j=(t, 
  x,\overline{g})$. Then

  \begin{align*}
    \frac{P_\delta\psi(j) - \psi(j)}{\delta}
    &=
      \frac{1}{\delta}\left(\ConditionExpect{\psi(J_{t+\delta})}{J_t =
      j} - \psi(j)\right)\\ 
    &= \ConditionExpect{\frac{1}{\delta}\left(\psi(J_{t+\delta}) -
      \psi(J_t)\right)}{J_t = j}\label{eq:proof:fd:expectation}\tag{$*$}\\
  \end{align*}

  We will proceed by applying \hyperref[app:ito]{\Ito's Lemma} to this
  expectation. However, we must first verify that $\indexedabove{t}{J}$
  satisfies the hypotheses of \Ito's Lemma, namely, it must be a
  \hyperref[app:martingale]{semimartingale}. It is easy to verify that this is
  the case. We will express the tuples $\xxJointFD_t = (t, X_t, 
  \xxTruncRet_t)\in\R_+\times\xxX\times\returnspace$ as $d+2$-dimensional vectors (since
  $\xxX\subset\R^d$), where the first $d$ dimensions
  encode the state $X_t$, the $d+1$th dimension encodes the truncated return 
  $\xxTruncRet_t$, and the last dimension encodes time. We have
  
  \begin{align*}
    M_t &\triangleq \begin{bmatrix}X_t - \Expect{X_t}\\0\\0\end{bmatrix}\\
    A_t &\triangleq \begin{bmatrix}\Expect{X_t}\\\xxTruncRet_t\\t\end{bmatrix}\\
    \xxJointFD_t &= M_t + A_t
  \end{align*}

  It follows immediately from Lemma \ref{lem:finite-variation} that
  $\indexedabove{t}{A}$ is a \hyperref[app:finite-variation]{finite variation
  process}. Furthermore, since $\indexedabove{t}{X}$ is a Feller-Dynkin
  process, we know from Lemma \ref{lem:martingale-generator} that $(X_t -
  \Expect{X_t})_{t\geq 0}$ is a \hyperref[app:martingale]{martingale}. Thus,
  $\indexedabove{t}{J}$ can be expressed as a sum of a
  \hyperref[app:martingale]{local martingale}\footnote{By the definition of a
  local martingale, given in Appendix \ref{app:martingale}, it is clear that all
  martingales are local martingales.} and a finite variation process, making it a
  semimartingale by definition.

  Since $\indexedabove{t}{J}$ is a semimartingale and $\psi\in
  C_0^2(\R_+\times\mathcal{X}\times\returnspace)$, we may apply \hyperref[app:ito]{\Ito's
  lemma} to expand \eqref{eq:proof:fd:expectation} as follows, where all expectations are conditioned on $J_t = j$,

  {\small %
  \begin{align*}
    \small
    \frac{(P_\delta - \identity)\psi(j)}{\delta}
    &= \frac{1}{\delta}\Expect{\int_t^{t+\delta}\sum_{i=1}^{d+2}
    \partialderiv{\psi(J_s)}{j^i}dJ_s^i +
      \frac{1}{2}\int_t^{t+\delta}\sum_{i=1}^{d+2}\sum_{k=1}^{d+2}
      \frac{\partial^2\psi(J_s)}{\partial j^i\partial j^k}d[J^i, J^k]_s}\\
    &= \partialderiv{}{t}\psi(j) + \overbrace{\frac{1}{\delta}\Expect{\int_t^{t+\delta}\sum_{i=1}^{d}\partialderiv{\psi(J_s)}{j^i}dJ_s^i
      +
      \frac{1}{2}\int_t^{t+\delta}\sum_{i=1}^{d}\sum_{k=1}^{d}\frac{\partial^2\psi(J_s)}{\partial
      j^i\partial j^k}d[J^i, J^k]_s}}^a\\
    &\qquad+
      \overbrace{\frac{1}{\delta}\Expect{\int_t^{t+\delta}\partialderiv{\psi(J_s)}{j^{d+1}}dJ^{d+1}_s
      + \frac{1}{2}\frac{\partial^2\psi(J_s)}{\partial (j^{d+1})^2}d[J^{d+1},J^{d+1}]_s}}^b\\
    &\qquad+
      \overbrace{\frac{1}{2\delta}\Expect{\int_t^{t+\delta}
      \sum_{i=1}^d\left(\frac{\partial^2\psi(J_s)}{\partial
      j^i\partial j^{d+1}}d[J^i, J^{d+1}]_s+\frac{\partial^2\psi(J_s)}{\partial
      j^i\partial j^{d+2}}d[J^i, J^{d+2}]_s\right)}}^c
  \end{align*}
  }

  Recall that $J^{1:d}_t = \proj{1}J_t = X_t$, and $J^{d+1}_t =
  \proj{2}J_t = \overline{G}_t$. In the limit as $\delta\downarrow 0$,
  the term $a$ above therefore is
  simply the generator of the process $\indexedabove{t}{X}$ applied to
  $\psi$. Moreover, since it was shown that
  $\indexedabove{t}{\overline{G}}$ is a finite variation process in
  Lemma \ref{lem:finite-variation}, it follows that $[J^i,
  J^{d+1}] = [J^i, J^{d+2}]\equiv 0$ for any $i\in \{1,\dots,d+1\}$
  \citep{le2016brownian}.\marc{What does $i \in [d+1]$ mean here? This sentence is confusing.} Consequently, we have $c\equiv
  0$. Simplifying,

  \begin{align*}
    \lim_{\delta\to 0}\frac{P_\delta\psi(j) - \psi(j)}{\delta}
    &= \mathscr{L}_X\psi(j) + \lim_{\delta\to
      0} \frac{1}{\delta}\ConditionExpect{ \int_t^{t+\delta}
      \partialderiv{\psi(J_s)}{\overline{g}}d\overline{G}_s}{J_t =
      j} + \partialderiv{\psi( j)}{t}\\
    &= \mathscr{L}_X\psi(j) +
      \partialderiv{\psi(j)}{\overline{g}}\gamma^tr(x) + \partialderiv{}{t}\psi(j)
  \end{align*}
  This completes the proof.
\end{proof}

\cbrkbe*
\begin{proof}
    Let $z \in \returnspace$. Let $\phi:(\mathcal{O}\times\returnspace)\to\R$ be given by $\phi((x, z')) = \indicator{z'\geq 0}$. Then define the function ${u:\R_+\times(\xxX\times\returnspace)\to\R}$ according to\marc{Fix: the argument to $\phi$ needs $x$.}
    \harley{Re: Conditional probability -- I believe the absolute continuity assumption prevents cases where condition has probability 0}
  \begin{align*}
      u(t, (x, z')) &= \ConditionExpect{\phi\big (x, \cbrprocess(z)_T\big)}{\xxJointFDCB_t(z)=(x,z')}\\
      &= \Pr\left(\gamma^{-T}(z - \xxTruncRet_T)\geq 0\bigg\vert\ \xxJointFDCB_t(z)=(x,z')\right)\\
      &= \Pr\left(z \geq \xxTruncRet_T\bigg\vert\ \xxJointFDCB_t(z)=(x,z')\right)\\
      &= \Pr\left(\gamma^t z' \geq \xxTruncRet_T-\xxTruncRet_t\bigg\vert\ X_t=x\right)\\
      &= \Pr\left(z'\geq\int_0^{T-t}\gamma^sr(X_{(t+s)\land T})ds\bigg\vert\ X_t=x\right)\\
      &= \Pr\left(z'\geq\int_0^T\gamma^sr(X_{(t+s)\land T})ds\bigg\vert\ X_t=x\right)\\
      &= \Pr(\xxG{x} \leq z')\\
      &= \cdf{\returnmeasure^\pi}(x, z')
  \end{align*}
  The conditional expectation and probabilities are well-defined by Assumption \ref{ass:method:density}. Note that $u$ has precisely the form of the solution to the Kolmogorov backward equation in Theorem \ref{thm:kbe}. Thus, Theorem \ref{thm:kbe} establishes that $\cdf{\returnmeasure^\pi}(x,\cdot)$ satisfies \eqref{eq:kbe} with the infinitesimal generator $\xxL_J$ of the conditional backward return process. Finally, since $\cdf{\returnmeasure^\pi}$ is time-homogeneous, its time derivative vanishes, and we are left with \eqref{eq:cbr-kbe}.
\end{proof}

\dhjb*
\begin{proof}
  Note that the term $\partialderiv{}{z}\cdf{\returnmeasure^\pi}(x,z)$ is the Radon-Nikodym derivative of $\returnmeasure^\pi(x)$ with respect to the Lebesgue measures. This derivative exists by Assumption \ref{ass:method:density}.
  We have, for any $z\in\returnspace$, $\xxTruncRet_t = z - \gamma^t\cbrprocess(z)_t$. Since $\cbrprocess(z)_t$ can be computed by a deterministic, differentiable transformation of $\xxTruncRet_t$ for any given $z\in\returnspace$, it follows that $\indexedabove{t}{\xxJointFDCB(z)}$ is a Feller-Dynkin process for each $z\in\returnspace$. 
  
  Denote the infinitesimal generator of $\indexedabove{t}{\xxJointFDCB(z)}$ by $\xxL_J$. The 
  generator exists since the conditional backward return process is a Feller-Dynkin process, as
  previously mentioned. By a change of variables we immediately see that $\xxL_J = 
  \xxL_G\rvert_{t=0} - \log\gamma\proj{2}\partialderiv{}{z}$, where $\xxL_G$ is the 
  infinitesimal generator of the truncated return process.
  
  By Lemma \ref{lem:generator}, we know that $\cdf{\returnmeasure^\pi}$ solves the Kolmogorov backward equation for the generator $\xxL_J$. Thus,

  \begin{align*}
    0 &= \xxL_J\cdf{\returnmeasure^\pi}(x, z)\\
    &= \xxL_G\cdf{\returnmeasure^\pi}(x, z) -
    z'\log\gamma\partialderiv{}{z'}\cdf{\returnmeasure^\pi}(x, z)\\
    &= \xxL_X\cdf{\returnmeasure^\pi}(x, z) - (r(x) + 
    z\log\gamma)\partialderiv{}{z}\cdf{\returnmeasure^\pi}(x, z)\\
  \end{align*}
\end{proof}

\shjb*
\begin{proof}
    Suppose $\cdf{\returnmeasure}$ satisfies \eqref{eq:dhjb:ito}. Then, making the substitution $\cdf{\returnmeasure}(x, z) = \Psi(\statistics(x), z)$ in \eqref{eq:dhjb:ito}, we have
    
    \begin{align*}
        0
        &= \langle\nabla_x\Psi(\statistics(x), z), \xxf(x)\rangle - (r(x) + z\log\gamma)\partialderiv{}{z}\Psi(\statistics(x), z)
        + \trace\left(\quadraticform{\xxs(x)}{\hessian{x}\Psi(\statistics(x), z)}\right)\\
        &= \nabla_{\statistics(x)}\Psi(\statistics(x), z)^\top\left(\jacobian_x\statistics(x)\right)\xxf(x) - (r(x) + z\log\gamma)\partialderiv{}{z}\Psi(\statistics(x), z)\\
        &\qquad+ \trace\left[\quadraticform{\xxs(x)}{\overbrace{\jacobian_x\left(\nabla_{\statistics(x)}\Psi(\statistics(x), z)^\top\jacobian_x\statistics(x)\right)}^{(a)}}\right]\\
    \end{align*}
    
    All of the differential quantities above exist almost everywhere due to Assumption \ref{ass:method:c2} and the hypothesis that $\Phi$ is statistically smooth. It remains only to compute $(a)$. We have
    
    \begin{align*}
        (a)
        &= \jacobian_x\left(\nabla_{\statistics(x)}\Psi(\statistics(x), z)^\top\jacobian_x\statistics(x)\right)\\
        &= \jacobian_x\statistics(x)^\top\left(\hessian{\statistics(x)}\Psi(\statistics(x), z)\right)\jacobian_x\statistics(x) + \nabla_{\statistics(x)}\Psi(\statistics(x), z)^\top\jacobian_x\jacobian_x\statistics(x)\\
        &= \quadraticform{\jacobian_x\statistics(x)}{\hessian{\statistics(x)}\Psi(\statistics(x), z)}
        + \sum_{k=1}^N\partialderiv{}{\proj{k}\statistics(x)}\Psi(\statistics(x), z)\hessian{x}\proj{k}\statistics(x)\\
        &= \quadraticform{\jacobian_x\statistics(x)}{\hessian{\statistics(x)}\Phi(\statistics(x), [V_{\min}, z])}
        + \sum_{k=1}^N\partialderiv{}{\proj{k}\statistics(x)}\Phi(\statistics(x), [V_{\min}, z])\hessian{x}\proj{k}\statistics(x)\\
        &= \statsdiffuse[\Phi](x, z) + \statediffuse[\Phi](x, z)
    \end{align*}
    
    Substituting this into $(a)$ above, we arrive at the desired result.
\end{proof}

\shjbcor*
\begin{proof}
    We consider the case where $\Phi$ imputes the statistics $\statistics(x)$ to a quantile distribution.
    Let $\phi:\xxX\times\returnspace\to\R$ be an arbitrary test function in the Schwartz class $\mathscr{S}$, and let $\returnmeasure = \Phi(\statistics(x))$ such that $\cdf{\returnmeasure}$ is a distributional solution to \eqref{eq:dhjb:ito}. 
    For brevity, denote $\mathcal{Y}=\xxX\times\returnspace$.
    Denote by $\vartheta:\R\to[0,1]$ the Heaviside step function $\vartheta(z) = \indicator{z>0}$. Then, we have that
    {\small
    \begin{align*}
        0
        &= \int_{\mathcal{Y}}\bigg[\phi(x, z)\left\langle\nabla_x\sum_{k=1}^N\vartheta(z - \proj{k}\statistics(x)), \xxf(x)\right\rangle - \phi(x, z)(r(x) + z\log\gamma)\partialderiv{}{z}\sum_{k=1}^N\vartheta(z - \proj{k}\statistics(x))\\
        &\qquad\qquad\qquad + \frac{1}{2}\phi(x, z)\trace\left(\quadraticform{\xxs(x)}{\left(\hessian{x}\sum_{k=1}^N\vartheta(z -\proj{k}\statistics(x))\right)}\right)\bigg]dzdx\\
        &= \int_{\mathcal{Y}}\bigg[\left\langle\phi(x, z)\nabla_x\sum_{k=1}^N\vartheta(z - \proj{k}\statistics(x)), \xxf(x)\right\rangle - \phi(x, z)(r(x) + z\log\gamma)\partialderiv{}{z}\sum_{k=1}^N\vartheta(z - \proj{k}\statistics(x))\\
        &\qquad\qquad\qquad + \frac{1}{2}\trace\left(\quadraticform{\xxs(x)}{\phi(x,z)\left(\hessian{x}\sum_{k=1}^N\vartheta(z -\proj{k}\statistics(x))\right)}\right)\bigg]dzdx\\
    \end{align*}
    }
    
    Taking distributional derivatives once, the Heaviside step functions are transformed into Dirac distributions, yielding
    \begin{align*}
        0
        &= \int_{\mathcal{Y}}\bigg[\left\langle-\phi(x, z)\sum_{k=1}^N\dirac{\proj{k}\statistics(x)}(z)\nabla_x\proj{k}\statistics(x), \xxf(x)\right\rangle - \phi(x, z)(r(x) + z\log\gamma)\sum_{k=1}^N\dirac{\proj{k}\statistics(x)}(z)\\
        &\qquad\qquad\qquad - \frac{1}{2}\trace\left(\quadraticform{\xxs(x)}{\phi(x,z)\left(\nabla_x\sum_{k=1}^N\dirac{\proj{k}\statistics(x)}(z)\right)\nabla_x\proj{k}\statistics(x)}\right)\bigg]dzdx\\
    \end{align*}
    Next, we carry out the second spatial derivative.
    {\small
    \begin{align*}
        0
        &= \int_{\mathcal{Y}}\bigg[\left\langle-\phi(x, z)\sum_{k=1}^N\dirac{\proj{k}\statistics(x)}(z)\nabla_x\proj{k}\statistics(x), \xxf(x)\right\rangle - \phi(x, z)(r(x) + z\log\gamma)\sum_{k=1}^N\dirac{\proj{k}\statistics(x)}(z)\\
        &\qquad\qquad - \frac{1}{2}\trace\left(\quadraticform{\xxs(x)}{\phi(x,z)\left(\nabla_x\sum_{k=1}^N\dirac{\proj{k}\statistics(x)}(z)\nabla_x\proj{k}\statistics(x)\right)}\right)\bigg]dzdx\\
        &= \int_{\mathcal{Y}}\bigg[\left\langle-\phi(x, z)\sum_{k=1}^N\dirac{\proj{k}\statistics(x)}(z)\nabla_x\proj{k}\statistics(x), \xxf(x)\right\rangle - \phi(x, z)(r(x) + z\log\gamma)\sum_{k=1}^N\dirac{\proj{k}\statistics(x)}(z)\\
        &\qquad\qquad - \frac{1}{2}\trace\left(\quadraticform{\xxs(x)}{\phi(x,z)\sum_{k=1}^N\bigg[\nabla_x\dirac{\proj{k}\statistics(x)}(z)\nabla_x\proj{k}\statistics(x) + \dirac{\proj{k}\statistics(x)}(z)\hessian{x}\proj{k}\statistics(x)\bigg]}\right)\bigg]dzdx\\
        &= \int_{\mathcal{Y}}\phi(x, z)\bigg[\left\langle\sum_{k=1}^N\dirac{\proj{k}\statistics(x)}(z)\nabla_x\proj{k}\statistics(x), \xxf(x)\right\rangle + (r(x) + z\log\gamma)\sum_{k=1}^N\dirac{\proj{k}\statistics(x)}(z)\\
        &\qquad\qquad\qquad + \frac{1}{2}\trace\left(\quadraticform{\xxs(x)}{\sum_{k=1}^N \dirac{\proj{k}\statistics(x)}(z)\hessian{x}\proj{k}\statistics(x)}\right)\bigg]dzdx\\
        &\qquad\qquad\qquad + \frac{1}{2} \overbrace{\int_{\mathcal{Y}}\trace\left(\quadraticform{\xxs(x)}{\phi(x, z)\nabla_x\dirac{\proj{k}\statistics(x)}(z)\nabla_x\statistics(x)}\right)dzdx}^{(a)}\\
    \end{align*}
    }
    
    We isolate the term $(a)$ as it involves the (distributional) derivative of the Dirac distribution, which is a strange object. However, since our equation holds for any test function $\phi$, we will show that, with the right choice of test function, $(a) = 0$.
    
    Choose any $\overline{x}\in\xxX$ and let $\epsilon>0$. Then let $\phi(x, z) = \varrho_\epsilon(x)\psi(z)$ where $\varrho_\epsilon:\xxX\to\R$ and $\psi:\returnspace\to\R$ are members of the Schwartz class $\mathscr{S}$.
    We define $\varrho_\epsilon(x)$ as follows,
    
    \begin{align*}
        \varrho_\epsilon(x) &= \frac{1}{\epsilon\sqrt{\pi}}\exp\left(-\frac{\norm{x - \overline{x}}^2}{\epsilon^2}\right)
    \end{align*}
    
    It is well known that $\varrho_\epsilon$ is a Schwartz function \citep{lax2002functional}. Moreover, since $\nabla_x\varrho_\epsilon(\overline{x}) = 0$ and $\varrho_\epsilon$ is smooth, we can find a neighborhood $B$ of $\overline{x}$ so small that $\sup_{x_1,x_2\in B}\norm{x_1 - x_2}\leq\epsilon$. We are left with
    
    \begin{align*}
        (a) &= \lim_{\epsilon\to 0}\bigg[\int_{B}\int_{\returnspace}\trace\left(\quadraticform{\xxs(x)}{\phi(x, z)\nabla_x\dirac{\proj{k}\statistics(x)}(z)\nabla_x\proj{k}\statistics(x)}\right)dzdx\\
        &\qquad\qquad+ \int_{\xxX\setminus  B}\int_{\returnspace}\trace\left(\quadraticform{\xxs(x)}{\phi(x, z)\nabla_x\dirac{\proj{k}\statistics(x)}(z)\nabla_x\proj{k}\statistics(x)}\right)dzdx\bigg]\\
        &= \lim_{\epsilon\to 0}\bigg[-\overbrace{\int_{B}\int_{\returnspace}\trace\left(\quadraticform{\xxs(x)}{\psi(z)\nabla_x\varrho_{\epsilon(x)}\dirac{\proj{k}\statistics(x)}(z)\nabla_x\proj{k}\statistics(x)}\right)dzdx}^{\mathcal{M}_\epsilon}\\
        &\qquad\qquad- \overbrace{\int_{\xxX\setminus B}\int_{\returnspace}\trace\left(\quadraticform{\xxs(x)}{\psi(z)\nabla_x\varrho_{\epsilon}(x)\dirac{\proj{k}\statistics(x)}(z)\nabla_x\proj{k}\statistics(x)}\right)dzdx}^{\mathcal{E}_\epsilon}\bigg]\\
    \end{align*}
    
    It is also well-known $\lim_{\epsilon\to 0}\varrho_\epsilon = \dirac{\overline{x}}$ \citep{lax2002functional}. Since necessarily $\overline{x}\not\in\xxX\setminus B$, the term $\mathcal{E}_\epsilon$ vanishes. Given that $\sup_{x_1,x_2\in B}\norm{x_1-x_2}\leq\epsilon$, we have
    
    \begin{align*}
        |\mathcal{M}_\epsilon|
        &\leq \epsilon\sup_{x\in B}\left|\int_{\returnspace}\trace\left(\quadraticform{\xxs(x)}{\psi(z)\dirac{\proj{k}\statistics(x)}(z)\nabla_x\proj{k}\statistics(x)}\right)dz\right|\\
        &= \epsilon\sup_{x\in B}\left|\trace\left(\quadraticform{\xxs(x)}{\psi(\proj{k}\statistics(x))\nabla_x\proj{k}\statistics(x)}\right)dz\right|\\
    \end{align*}
    
    By the assumption that $\statistics(x)$ is almost-everywhere differentiable, the supremum above is bounded for almost every $\overline{x}$, and it follows that $|\mathcal{M}_\epsilon|\to 0$ almost surely.
    
    We are left with the following equation:
    
    \begin{equation*}
    \begin{aligned}
      0
      = \lim_{\epsilon\to 0}\int_{\xxX}\int_{\returnspace}\varrho_\epsilon(x)\psi( z)\sum_{k=1}^N\dirac{\proj{k}\statistics(x)}(z)\bigg[\langle\nabla_x\proj{k}\statistics(x), \xxf(x)\rangle + r(x) + z\log\gamma\qquad &\\
      \qquad+\frac{1}{2}\trace\left(\quadraticform{\xxs(x)}{\hessian{x}\proj{k}\statistics(x)}\right)\bigg]dzdx&
      \end{aligned}
    \end{equation*}
    
    Given that $\Phi(\statistics(x))$ is statistically smooth, it is a tempered distribution, so this limit exists. We mentioned previously that $\varrho_\epsilon\to\dirac{\overline{x}}$, so we have
    
    \begin{equation*}
    \begin{aligned}
      0
      = \int_{\returnspace}\psi(z)\sum_{k=1}^N\dirac{\proj{k}\statistics(\overline{x})}(z)\bigg[\langle\nabla_x\proj{k}\statistics(\overline{x}), \xxf(\overline{x})\rangle + r(\overline{x}) + z\log\gamma\qquad&\\
      +\frac{1}{2}\trace\left(\quadraticform{\xxs(\overline{x})}{\hessian{x}\proj{k}\statistics(\overline{x})}\right)\bigg]dz&
      \end{aligned}
    \end{equation*}
    
    It follows by definition that $\Phi(\statistics(x))$ is a distributional solution to
    
    \begin{equation*}
    \begin{aligned}
      0
      &= \sum_{k=1}^N\dirac{\proj{k}\statistics(\overline{x})}(z)\bigg[\langle\nabla_x\proj{k}\statistics(\overline{x}), \xxf(\overline{x})\rangle + r(\overline{x}) + z\log\gamma
      +\frac{1}{2}\trace\left(\quadraticform{\xxs(\overline{x})}{\hessian{x}\proj{k}\statistics(\overline{x})}\right)\bigg]
      \end{aligned}
    \end{equation*}
    
    Note that the equation above is a sum of weighted Diracs. Thus, the only way for it to be satisfied is if each of the terms in the sum individually vanishes. So, we have shown that for each $k\in[N]$ and almost every $x\in\xxX$, the statistics function $\proj{k}\statistics$ is a distributional solution of
    
    \begin{align*}
        0 &= \langle\nabla_x\proj{k}\statistics(x), \xxf(x)\rangle + r(x) + \proj{k}\statistics(x)\log\gamma + \frac{1}{2}\trace\left(\quadraticform{\xxs(x)}{\hessian{x}\proj{k}\statistics(x)}\right)
    \end{align*}
    
    This completes the proof.
\end{proof}

\subsection{Solution of the Kolmogorov Backward Equation}\label{app:proofs:kbe}
Recall the identity presented about the solution of the Kolmogorov
Backward Equation as an expectation,

\kbe*

In order to prove Theorem \ref{thm:kbe}, the following lemma will be
handy.

\begin{lemma}[\citep{le2016brownian}, Theorem 6.14]\label{lem:martingale-generator}
  Let $(X_t)_{t\geq 0}$ be a Feller-Dynkin process on a metric space
  $\mathcal{X}$, and consider functions $h, g\in
  C_0(\mathcal{X})$. The following two conditions are equivalent:
  \begin{enumerate}
  \item $h\in\mathscr{D}(\mathscr{L})$ and $\mathscr{L}h = g$;
  \item For each $x\in\mathcal{X}$, the process
    \begin{equation*}
      \Conditional{h(X_t) - \int_0^tg(X_s)ds}{X_0 = x}
    \end{equation*}
    is a \hyperref[app:martingale]{martingale} with respect to
    the filtration $(\mathcal{F}_t)$.
  \end{enumerate}
\end{lemma}

\begin{proof}[Proof of Theorem \ref{thm:kbe}]
  By Lemma \ref{lem:martingale-generator}, we know that the process
  $\Phi_t = \phi(X_t) -  \int_0^tg(X_s)ds$ is a martingale with
  respect to $(\mathcal{F}_t)$. Let $s<t<T$. By the definition of a
  martingale, we have

  \begin{equation*}
    \small
    \begin{aligned}
      0 &= \Expectation{}{\Conditional{\Phi_T}{\mathcal{F}_t}} -
      \Expectation{}{\Conditional{\Phi_T}{\mathcal{F}_s}}\\
      &=
    \Expectation{}{\Conditional{h(X_T) +
        \int_0^Tg(X_r)dr}{\mathcal{F}_t}} -
    \Expectation{}{\Conditional{h(X_T) + \int_0^Tg(X_r)dr}{\mathcal{F}_s}}\\
    \Expectation{}{\Conditional{\int_s^t\mathscr{L}h(X_r)dr}{\mathcal{F}_t}}
    &=
    \Expectation{}{\Conditional{h(X_T)}{\mathcal{F}_t}} -
    \Expectation{}{\Conditional{h(X_T)}{\mathcal{F}_s}}
    \end{aligned}
  \end{equation*}

  Dividing through by $t - s$ and taking the limit as $s\uparrow t$,

  \begin{equation*}
    \begin{aligned}
      \partialderiv{}{s}\Expectation{}{\Conditional{\phi(X_T)}{\mathcal{F}_s}}
      = \partialderiv{}{s}u(x, s)
      &\overset{(a)}{=}
      \Expectation{}{\Conditional{\partialderiv{}{s}\int_s^t\mathscr{L}\phi(X_r)dr}{\mathcal{F}_t}}\\
      &= -\Expectation{}{\Conditional{\mathscr{L}\phi(X_r)dr}{\mathcal{F}_s}}\\
      &\overset{(b)}{=}-\mathscr{L}\Expectation{}{\Conditional{\phi(X_s)}{\mathcal{F}_s}}\\
      &= -\mathscr{L}u(x, s)
    \end{aligned}
  \end{equation*}

  Step $(a)$ is allowed by the Leibniz integration rule since the
  infinitesimal generator preserves continuity and $\phi$ is
  absolutely continuous by assumption. Finally, step $(b)$ is allowed
  by the linearity of expectation, since $\mathscr{L}$ is a linear operator.
\end{proof}

\section{Further Experiment Details}\label{app:experiments}
Figure \ref{fig:birdseye} below demonstrates that the continuous-time algorithm does
indeed learn more accurate representations of the return distribution function than QTD.

\begin{figure}[h]
  \centering
  \newcommand{\figwidth}{0.495\textwidth}
  \newcommand{\birdseyescale}{0.74}
  QTD\\
  \begin{minipage}{\figwidth}
  \centering
    \includegraphics[scale=\birdseyescale]{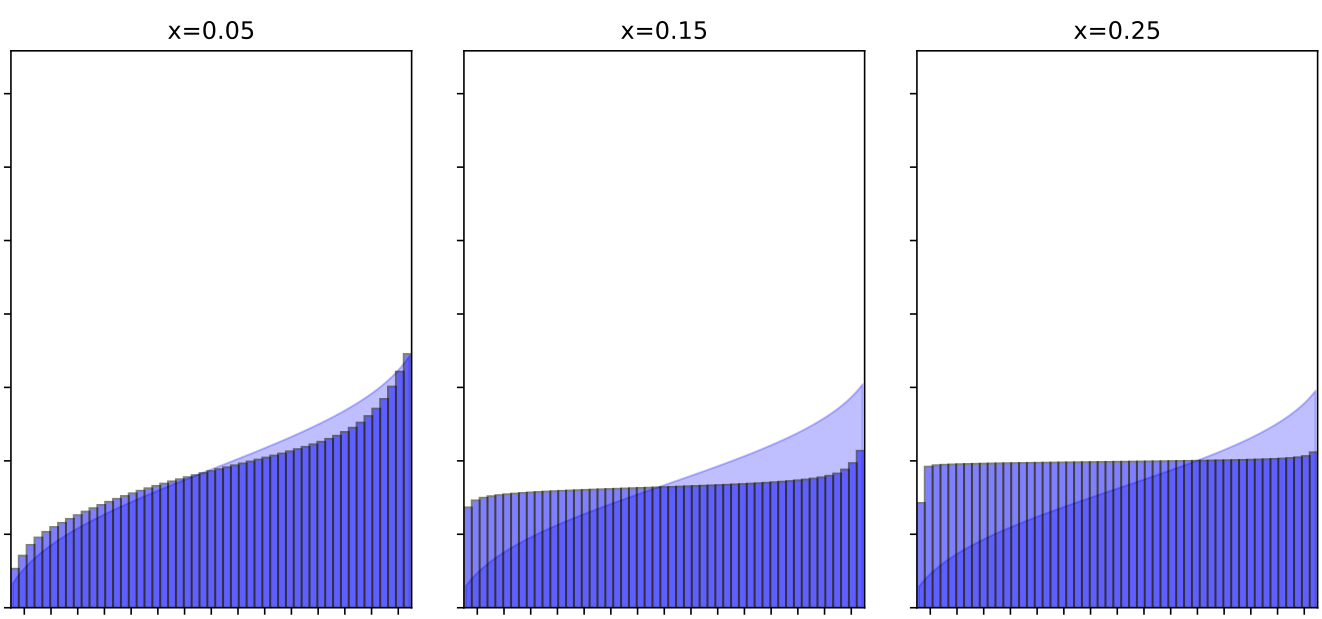}
  \end{minipage}
  \begin{minipage}{\figwidth}
  \centering
    \includegraphics[scale=\birdseyescale]{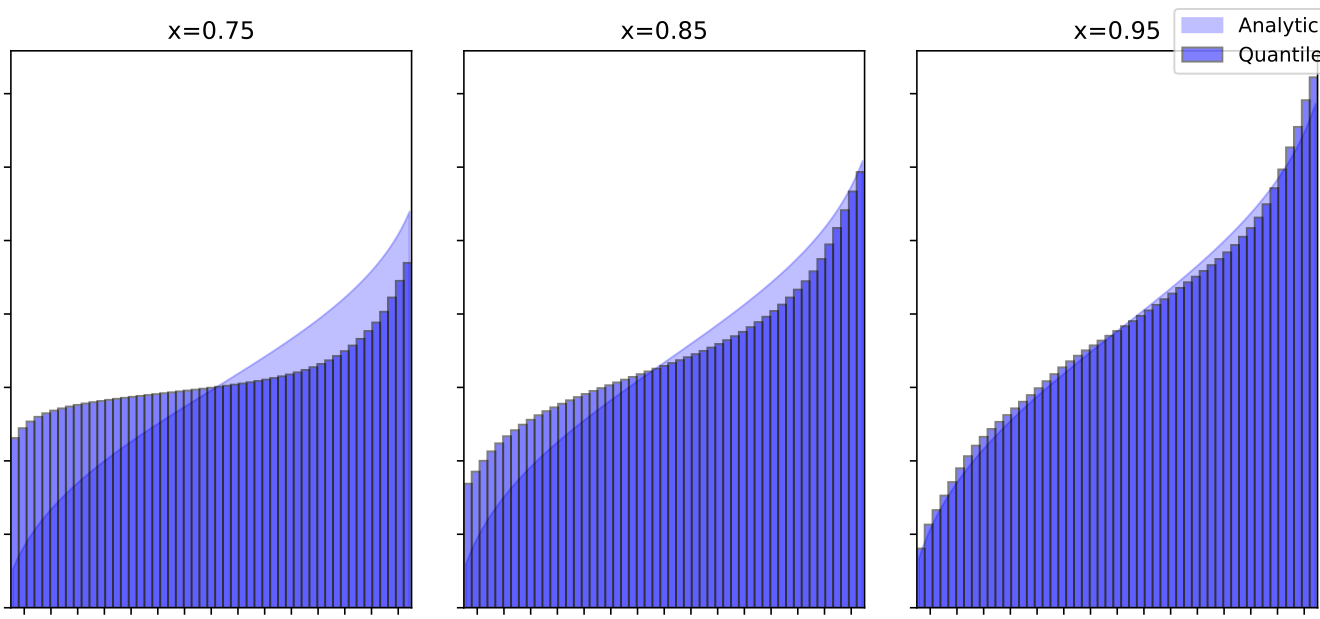}
  \end{minipage}
  \\\vspace{0.5cm}
  FD-WGF Q-Learning\\
  \begin{minipage}{\figwidth}
  \centering
    \includegraphics[scale=\birdseyescale]{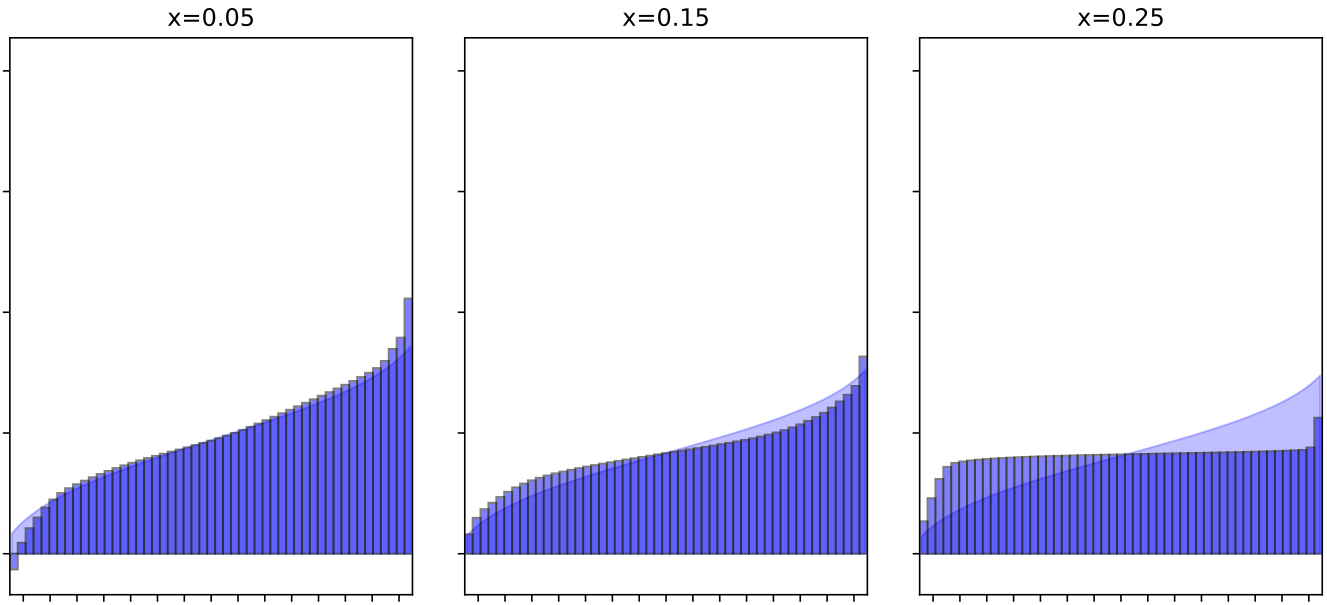}
  \end{minipage}
  \begin{minipage}{\figwidth}
  \centering
    \includegraphics[scale=\birdseyescale]{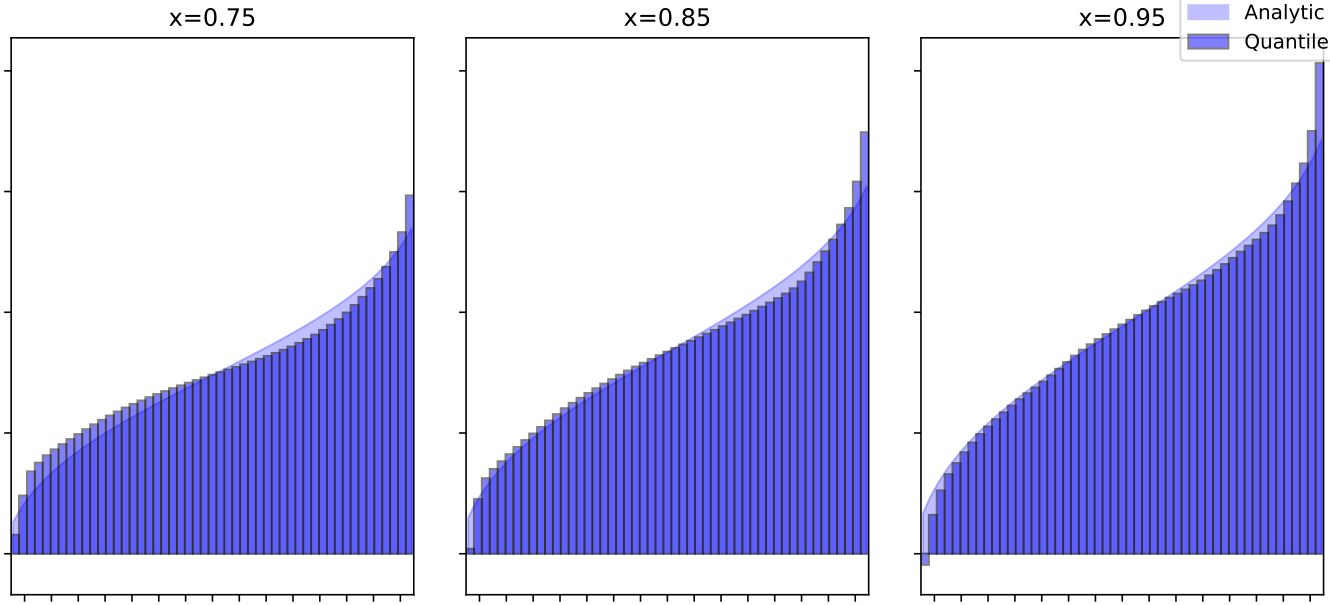}
  \end{minipage}
  \caption{Quantile functions learned for the toy problem}
  \label{fig:scan}
\end{figure}
\section{Tools from the Theory of Stochastic Processes}\label{app:stochastic}
This appendix will survey some concepts from the theory of stochastic processes
that are useful in the developments of this work.

\subsection{Some Special Classes of Stochastic Processes}
\subsubsection{Measurable, Adapted, and Progressive Processes}\label{app:adapted}
When dealing with stochastic processes, there are a few properties that we
generally desire in order for us to be able to analyze them nicely. The most
common examples will be summarized here. These definitions are due to
\citet{le2016brownian}.

For the following definitions, we will fix a
probability space $(\Omega, \mathcal{F},
\Pr)$, and we will consider a stochastic process
$\indexedabove{t}{X}\subset\mathcal{X}$, where $(\mathcal{X}, \Sigma)$ is a
measurable space.

\begin{definition}[Measurable Process]\label{def:process:measurable}
  The process
  $\indexedabove{t}{X}\subset \mathcal{X}$
  is said to be
  \emph{measurable} if $(\omega, t)\mapsto X_t(\omega)$ is a measurable map on
  $\Omega\times\mathbf{R}_+$ with respect to the smallest $\sigma$-algebra on
  $\mathscr{B}(\mathbf{R}_+)\times\mathcal{F}$.
\end{definition}

For the remainder of the definitions, we will also consider a
\hyperref[def:filtration]{filtration} (see Definition \ref{def:filtration})
$\indexedabove{t}{\mathcal{F}}$ making $(\Omega, \mathcal{F},
\indexedabove{t}{\mathcal{F}},\Pr)$ a filtered probability space.

\begin{definition}[Adapted Process]\label{def:process:adapted}
  The process $\indexedabove{t}{X}\subset\mathcal{X}$ is \emph{adapted} if $X_t$
  is $\mathcal{F}_t$-measurable for every $t\geq 0$.
\end{definition}

\begin{definition}[Progressive Process]\label{def:process:progressive}
  The process $\indexedabove{t}{X}\subset\mathcal{X}$ is \emph{progressive}
  (or \emph{progressively measurable}) if $(\omega, s)\mapsto X_t(\omega)$ is
  measurable on $\Omega\times[0,t]$ with respect to the smallest
  $\sigma$-algebra on $\mathcal{F}_t\times\mathscr{B}([0,t])$ for each $t\geq
  0$.
\end{definition}

\subsubsection{Martingales}\label{app:martingale}
\begin{definition}[Martingales, \cite{rogers1994diffusions}]
  \label{def:martingale}
  A \textbf{martingale} (relative to a given
  \hyperref[def:filtration]{filtration} $(\mathcal{F}_t)_{t\geq 0}$)
  is a stochastic process $(M_t)_{t\geq 0}$ where 
  $M_t\in L^1$ and
  \begin{equation}
    \label{eq:martingale:property}
    M_s = \ConditionExpect{M_t}{\mathcal{F}_s}\qquad 0\leq s \leq t
  \end{equation}

  Equation \eqref{eq:martingale:property} is referred to as ``the
  martingale property''. If the equality in
  \eqref{eq:martingale:property} is instead $\geq$ (resp. $\leq$), $(M_t)_{t\geq 0}$
  is called a \textbf{supermartingale} (resp. \textbf{submartingale}).
\end{definition}

\begin{definition}[Local Martingales, \cite{le2016brownian}]
  \label{def:local-martingale}
  A \textbf{local martingale} is a stochastic process $(M_t)_{t\geq
    0}$ for which there exists a sequence of nondecreasing
  \hyperref[def:stopping-time]{stopping times} $(T_n)_{n=1}^\infty$
  such that $M^{T_n} = (M_{t\land T_n})_{t\geq 0}\in L^1$ is a martingale.
\end{definition}

\begin{definition}[Semimartingales, \cite{le2016brownian}]\label{def:semimartingale}
  A \textbf{semimartingale} is a random process $(X_t)_{t\geq 0}$ such
  that $X_t = A_t + M_t$ for each $t\geq 0$, where $(A_t)_{t\geq 0}$
  is a \hyperref[app:finite-variation]{finite variation process} and
  $(M_t)_{t\geq 0}$ is a local martingale.
\end{definition}

\subsubsection{Finite Variation Processes}\label{app:finite-variation}
\begin{definition}[Finite Variation Function, \cite{le2016brownian}]
  Let $T\geq 0$. A continuous function $a : [0,T]\to\mathbf{R}$ with
  $a(0) = 0$ is said to have \textbf{finite variation} if there exists
  a signed measure $\mu$ on $[0,T]$ such that $a(t) = \mu([0,t])$ for
  any $t\in [0,T]$.
\end{definition}

A finite variation process is a process whose regularity is given by
finite variation sample paths, as formalized in the next definition.

\begin{definition}[Finite Variation Process, \cite{le2016brownian}]
  A process $(A_t)_{t\geq 0}$ is called a \textbf{finite variation
    process} if all of its sample paths are finite variation functions
  on $\mathbf{R}_+$.
\end{definition}

The following processes generalize the notion of covariance of random variables
to stochastic processes, and appear frequently in important stochastic calculus
theorems. Their definitions are given by \citet{le2016brownian}.

\begin{definition}[Quadratic Variation]\label{def:quadratic-variation}
  Let $\indexedabove{t}{M}$ be a \hyperref[def:local-martingale]{local
  martingale}. The \emph{quadratic variation} of $\indexedabove{t}{M}$,
  denoted $\indexedabove{t}{[M,M]}$, is the unique increasing process such
  that $(M^2_t - [M,M]_t)_{t\geq 0}$ is a local martingale.
\end{definition}

\begin{remark}
  The existence and uniqueness of the quadratic variation is shown by
  \citet[Theorem 4.9]{le2016brownian}.
\end{remark}

\begin{definition}[The Bracket of Local Martingales]\label{def:bracket}
  Let $\indexedabove{t}{M},\indexedabove{t}{N}$ be local martingales. The
  \emph{bracket} of $M,N$, denoted $\indexedabove{t}{[M,N]}$ is the finite
  variation process $\indexedabove{t}{[M,N]}$ given by
  \begin{align*}
    [M,N]_t &= \frac{1}{2}\bigg([M+N,M+N]_t - [M,M]_t - [N,N]_t\bigg)
  \end{align*}
\end{definition}

\subsection{\Ito's Lemma}
\Ito's Lemma is a very powerful tool in the analysis of stochastic
processes. It can be thought of as a stochastic analog to Taylor's theorem.

\begin{theorem}[\Ito's Lemma, \cite{le2016brownian}]\label{app:ito}
  Let $(X^i)_{i=1}^p$ be real valued
  \hyperref[def:semimartingale]{semimartingales} and let $f\in
  C^2(\mathbf{R})$. Let $\mathbf{X}_t = (X_t^1,\dots,X_t^p)$. Then, for every $t\geq 0$,

  \begin{equation}
    \begin{aligned}
    \label{eq:ito}
    f(\mathbf{X}_t) = f(\mathbf{X}_0) +
    \sum_{i=1}^p\int_0^t\partialderiv{f}{x^i}(\mathbf{X}_s)dX_s^i +
    \frac{1}{2}\sum_{i=1}^p\sum_{j=1}^p\int_0^t\frac{\partial^2f}{\partial
      x_i\partial x_j}(\mathbf{X}_s)d[X^i, X^j]_s
    \end{aligned}
  \end{equation}
\end{theorem}

\section{Continuous-Time Markov Processes}\label{app:ctmp}

While discrete-time Markov processes are common in the reinforcement
learning literature, continuous-time Markov processes (particularly
\emph{stochastic} continuous-time Markov processes) are not a trivial extrapolation.

To begin, we recall the definition of a \emph{filtration}, which
extends the notion of a $\sigma$-algebra to time-dependent random
variables (i.e., stochastic processes).

\begin{definition}[Filtration,
  \cite{le2016brownian}]\label{def:filtration}
  Let $(\Omega, \mathcal{F}, \Pr)$ be a probability space. A
  \emph{filtration} of $\mathcal{F}$ is a collection
  $(\mathcal{F}_t)_{t\geq 0}$ of
  $\sigma$-algebras where $\mathcal{F}_t\subset\mathcal{F}$
  for each $t$, and $\mathcal{F}_s\subset\mathcal{F}_t$ whenever
  $s<t$. A probability space associated with a filtration is called a
  \emph{filtered probability space}, and is written as the 4-tuple
  $(\Omega, \mathcal{F}, (\mathcal{F}_t)_{t\geq 0}, \Pr)$.
\end{definition}

\begin{definition}[Canonical Filtration, \cite{le2016brownian}]\label{def:canonical-filtration}
    Let $\indexedabove{t}{X}$ be a stochastic process on a probability space $(\Omega, \mathcal{F}, \Pr)$. The \emph{canonical filtration} is a filtration $\indexedabove{t}{\mathcal{F}}$ where $\mathcal{F}_t$ is the $\sigma$-algebra generated by all observations of the process $\indexedabove{t}{X}$ occuring at or before time $t$.
\end{definition}

A Markov process can then be defined as a stochastic process on a
filtered probability space that satisfies a Markov property.

\marc{Fix $X$ to not be the state space. Define norm.}
\begin{definition}[Markov Process,
  \cite{rogers1994diffusions}]\label{def:transition-semigroup}
  Let $(X_t)_{t\geq 0}$ be a stochastic process in the
  \hyperref[def:filtration]{filtered probability space} $(\Omega,
  \mathcal{F}, (\mathcal{F}_t)_{t\geq 0}, \Pr)$. A \emph{Markovian
    transition kernel}
  $P_t:\Omega\times\mathcal{F}\to[0,1]$
  is a transition kernel with a continuous parameter $t$,
  such that for any bounded
  $\mathscr{B}(\mathbf{R}_+)\otimes\mathcal{F}$-measurable function $f$, we have

  \begin{equation}\label{eq:markov-kernel}
    (P_tf)(s, X_s) = \ConditionExpect{f(s+t, X_{s+t})}{\mathcal{F}_s}\qquad
                                                                  \Pr-\text{almost surely}
  \end{equation}

  A collection $(P_t)_{t\geq 0}$ of Markovian transition kernels is called a
  \emph{transition semigroup}\footnote{This name emphasizes
    the semigroup nature of the collection of transition kernels. In
    the abstract algebra literature, a semigroup is a set of objects that is
  closed under an associative binary operation.} when
  \begin{enumerate}
  \item For each $t\geq 0$ and $x\in\Omega$, $P_t(x,\cdot)$ is a
    measure on $\mathcal{F}$ and $P_t(x,\Omega)\leq 1$;
  \item For each $t\geq 0$ and $\Gamma\in\mathcal{F}$, the mapping
    $P_t(\cdot,\Gamma)$ is $\mathcal{F}$-measurable; and
  \item (The Chapman-Kolmogorov Identity) For each $s,t\geq 0$,each
    $x\in\Omega$, and each $\Gamma\in\mathcal{F}$, the collection
    satisfies

    \begin{align*}
    P_{s+t}(x,
    \Gamma) = \int_\Omega P_s(x, dy)P_t(y, \Gamma)
    \end{align*}
  \end{enumerate}
  Then $P_tP_s = P_{t+s}$, so $\indexedabove{t}{P}$ is indeed a semigroup.

  A \emph{Markov process} is a stochastic process $\indexedabove{t}{X}$ together
  with a transition semigroup $\indexedabove{t}{P}$ such that
  \eqref{eq:markov-kernel} holds.
\end{definition}

Markov processes with smooth transition kernels are often
desirable. This notion is formalized by the following concept.

\begin{definition}[Feller-Dynkin Process, Infinitesimal Generator,
  \cite{rogers1994diffusions}]\label{def:fd}
  Consider a filtered probability space $(\Omega, \mathcal{F},
  (\mathcal{F}_t)_{t\geq 0}, \Pr)$ and let $\mathcal{X}$ be a Polish\footnote{A
    Polish space is a complete
    metric space that has a countable, dense
    subset.} space. A transition semigroup
  $(P_t)_{t\geq 0}$ is said to be a \emph{Feller semigroup} if
  \begin{enumerate}
  \item $P_t : C_0(\mathcal{X})\to C_0(\mathcal{X})$ for each
    $t\in\mathbf{R}_+$;
  \item For any $f\in C_0(\mathcal{X})$ with $f\leq 1$, $P_tf\in[0,1]$;
  \item $P_sP_t = P_{s+t}$ and $P_0=\identity$;
  \item For any $f\in C_0(\mathcal{X})$, we have
    $\|P_tf-f\|\overset{t\downarrow 0}{\longrightarrow} 0$.
  \end{enumerate}

  A Markov process with a Feller semigroup is called a
  \emph{Feller-Dynkin process}.

  Define the set $\mathscr{D}(\mathscr{L})$ according to
  \begin{align*}
    \mathscr{D}(\mathscr{L}) = \bigg\{\Conditional{f\in
                               C_0(\mathcal{X})}{&\exists g\in
                               C_0(\mathcal{X})\quad\text{\small such
                               that}\\
    &\quad \left\|\frac{P_\delta - f}{\delta}
                               - g\right\|\overset{\delta\downarrow
                               0}{\longrightarrow} 0}\bigg\}
  \end{align*}

  The \emph{infinitesimal generator} of a Feller-Dynkin process is the
  operator $\xxL:\mathscr{D}(\xxL)\to C_0(\mathcal{X})$ where
  \begin{align*}
    \xxL f = \lim_{\delta\to 0}\frac{P_\delta f - f}{\delta}
  \end{align*}

  and $\mathscr{D}(\xxL)$ is called the \emph{domain of the
    infinitesimal generator} $\mathscr{\xxL}$.
\end{definition}

To deal with non-deterministic times in the analysis of a
continuous-time Markov process, we recall the formalism of a
\emph{stopping time}.

\begin{definition}[Stopping time, \citep{le2016brownian}]\label{def:stopping-time}
  Let $(\Omega, \mathcal{F}, (\mathcal{F}_t))$ be a measurable space
  with \hyperref[def:filtration]{filtration} $(\mathcal{F}_t)$. A random variable
  $T:\Omega\to\mathbf{R}_+$ is called a \emph{stopping time} with
  respect to the filtration $(\mathcal{F}_t)$ if
  \begin{equation*}
    \{T\leq t\} \in\mathcal{F}_t\qquad t\geq 0
  \end{equation*}

  We define the \emph{$\sigma$-algebra of the past before $T$} as the
  $\sigma$-algebra $\mathcal{F}_T$ given by

  \begin{equation*}
    \mathcal{F}_T = \left\{A\in\mathcal{F}_\infty : A\cap\{T\leq t\}\in\mathcal{F}_t\right\}
  \end{equation*}
\end{definition}

\subsection{Brownian Motion}\label{app:ctmp:brownian}
Brownian motion is ubiquitous in the study of stochastic
processes. The idea can be motivated as follows.

Let $X_0\triangleq 0\in\mathbf{R}$. Suppose we are modeling the
trajectory of the random process $\indexedabove{t}{X}$, where $X$ is
``continuously perturbed'' by Gaussian noise with mean $0$. What does
it mean for something to be \emph{continuously perturbed} by noise? A
natural way to reason about this is to discretize time, and suppose
that the variable at consecutive timesteps differs by a random
quantity sampled independently from a Gaussian with zero
mean. We want $X_1$ to have variance $1$, and we want this variance to
spread evenly through time in the sense that $X_t$ has variance
$t$. We can begin with a very coarse discretization where the timestep
$\tau$ has duration $1$, which involves sampling
$X_1\sim\gaussian{0}{1}$ and interpolate linearly form $t=0$ to
$t=1$. Then we can study the behavior as $\tau\to 0$. For any
$\tau>0$, we simply sample $X_{t+\tau}\sim X_t +
\gaussian{0}{\tau}$. Alternatively, we can sample 
$(X_{k\tau})_{k\in\mathbf{N}}$ via a
Gaussian process with covariance kernel $K(X_s, X_t) = \min(s, t)$
\citep{williams2006gaussian}. Figure \ref{fig:brownian:viz:1}
illustrates some of these samples for various values of $\tau$.

\begin{figure}[h]
  \centering
  \includegraphics[scale=0.8]{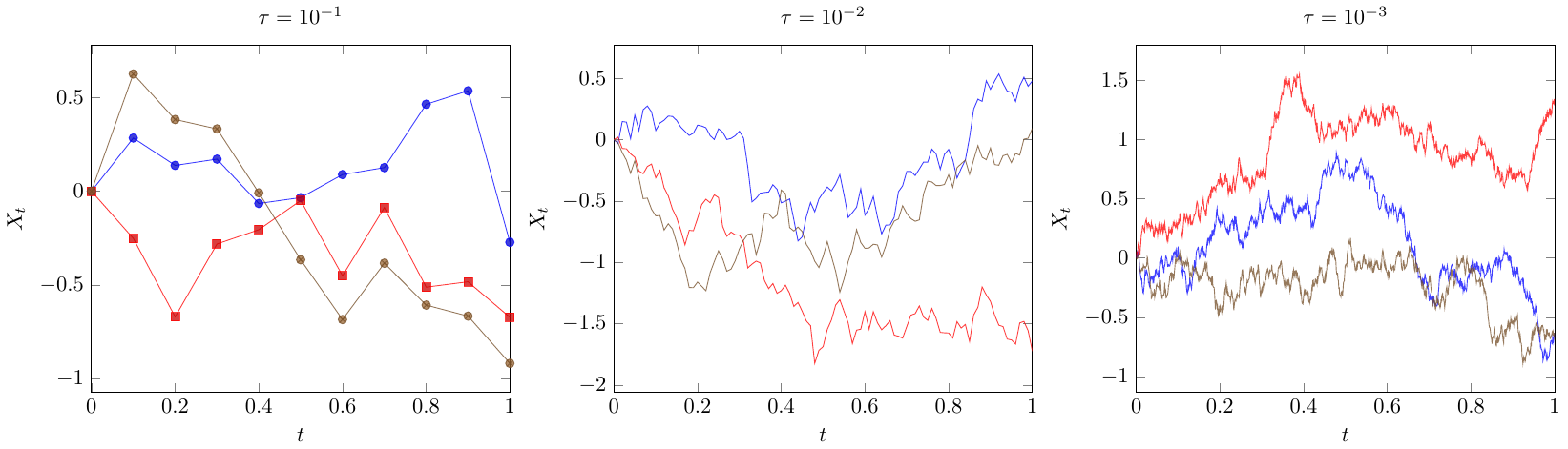}
  \caption{Discretized Brownian motion trajectories for various
    timesteps $\tau$}
  \label{fig:brownian:viz:1}
\end{figure}

Considering once again the filtered probability space $(\Omega,
\mathcal{F}, \indexedabove{t}{\mathcal{F}}, \mu)$, the criteria for a
Brownian motion $\indexedabove{t}{B}$ can be stated formally as

\begin{enumerate}
\item $B_0=0$, $\mu$-almost surely;
\item For any $0\leq r<s<t$, the random variable $B_t - B_s$ is
  independent from $\mathcal{F}_r$ and is distributed according to
  $\gaussian{0}{t-s}$;
\item The \emph{sample paths} of $\indexedabove{t}{B}$, defined as the
  mappings $t\mapsto B_t(\omega)$ for any fixed
  $\omega\in\mathcal{F}_t$, are continuous.
\end{enumerate}

Proving that such a process exists is not trivial by any
means. Fortunately, Brownian motion \emph{does} exist, and
\citet{le2016brownian} can be consulted for its construction.
\section{The Feynman-Kac Formula}\label{app:feynman-kac}
We make use of the following formulation of the \emph{Feynman-Kac
  formula}, as illustrated in \citet[Exercise 6.26]{le2016brownian}.

\begin{theorem}\label{thm:feynman-kac}
  Let $\indexedabove{t}{X}$ be a
  \hyperref[def:fd]{Feller-Dynkin} process in a space $\mathcal{X}$
  and let $v\in C_0(\mathcal{X})$. Define for any $x\in\mathcal{X}$
  and $\phi$ a bounded and measurable function over $\mathcal{X}$ the
  transition semigroup $\indexedabove{t}{Q^\star}$ where

  \begin{align*}
    Q_t^\star\phi(x) &= \ConditionExpect{\phi(X_t)\exp\left(-\int_0^tv(X_s)ds\right)}{X_0=x}
  \end{align*}

  If $\indexedabove{t}{X}$ admits an infinitesimal generator
  $\mathscr{L}$ and $\phi\in\mathcal{D}({\mathscr{L}})$, then

  \begin{equation}
    \label{eq:feynman-kac:generator}
    \frac{d}{dt}Q_t^\star\phi\rvert_{t=0} = \mathscr{L}\phi - v\otimes\phi
  \end{equation}
\end{theorem}

\begin{remark}
  The Feynman-Kac formula can be seen as the
  \hyperref[thm:kbe]{Kolmogorov Backward Equation} with an
  ``integrating factor''. Effectively, the Feynman-Kac formula
  allows us to identify solutions of PDEs of the form

  \begin{align*}
    \partialderiv{u}{t} &= - \mathscr{L}u + v\otimes\phi
  \end{align*}

  with conditional expectations of diffusion processes.
\end{remark}

\section{Wasserstein Gradient Flows}\label{app:wgf}
Recall that in continuous time,
the value function is characterized by a PDE. We should therefore
anticipate that the return distribution function will also be
characterized by some differential equation. In discrete-time RL
algorithms the value function is updated to minimize its difference to
the fixed point of the Bellman operator. In the continuous-time limit,
this is represented by a \emph{gradient flow} \citep{Santambrogio2016EuclideanMA},
\begin{equation}
  \label{eq:cauchy}
  \partialderiv{}{t}\returnmeasure_t =
  -\nabla\mathscr{G}\returnmeasure_t
\end{equation}
where $\mathscr{G}$ is a ``loss functional'' that effectively computes
the distance between $\returnmeasure_t$ and its fixed point. However,
\eqref{eq:cauchy} has some glaring problems: the space of probability
measures is not a vector space, so neither of the terms in
\eqref{eq:cauchy} are meaningful. To cope with this, we will consider
an alternate form of \eqref{eq:cauchy} that can be expressed entirely
in terms of metric space properties, called the \emph{Evolution
  Variational Inequality} \citep{de1980problems},
\begin{equation}
  \label{eq:evi}\tag{$\text{EVI}_\lambda$}
  \frac{1}{2}\partialderiv{}{t}d^2(\mu_t, \nu) \leq \mathscr{G}(\nu) -
  \mathscr{G}(\mu_t) + \frac{\lambda}{2}d^2(\mu_t, \nu)
\end{equation}
where $\lambda > 0$ and $\mu_t, \nu$ are elements of an abstract
metric space with metric $d$. When the metric space is Euclidean,
\eqref{eq:evi} and \eqref{eq:cauchy} are equivalent
\citep{Santambrogio2016EuclideanMA}. This characterization of a
gradient flow is much more attractive considering the following
result.

\begin{theorem}[\citet{muratori2018gradient}, Theorem 3.5]
  Let $(\mathcal{V}, d)$ be a metric space and suppose
  $\mathscr{G}:\mathcal{V}\to\R_+$ is $\lambda$-convex. If two
  curves $\mu, \nu:\R_+\to\mathcal{V}$ satisfy \eqref{eq:evi}, then
  \begin{align*}
    d(\mu_t, \nu_t) &\leq e^{-\lambda t}d(\mu_0, \nu_0)
  \end{align*}
  Consequently, for any given initial data $\mu_0 = \varrho$,
  solutions to \eqref{eq:evi} must be unique.
\end{theorem}

The machinery of abstract gradient flows has been particularly
fruitful in the analysis of curves in $2$-Wasserstein space. The
celebrated work of \citet{Jordan02thevariational} establishes an
equivalence between such a Wasserstein gradient flow (WGF) and the
\emph{Fokker-Planck equation},
\begin{equation}
  \label{eq:fokker-planck}\tag{FP}
  \partialderiv{}{t}\varrho_t(x) = -\nabla\cdot(\varrho_t(x)f(x)) +
  \beta\Delta\varrho_t(x)
\end{equation}
whose solution is the density of the solution to the stochastic
differential equation given by
\begin{align*}
  dX_t &= f_t(X_t)dt + \sqrt{\beta}dB_t
\end{align*}
where $\beta\in\R_+$ and $\indexedabove{t}{B}$ is a Brownian motion
\citep{ambrosio2008gradient}. Ultimately,
\citet{Jordan02thevariational} introduces a time-discretized scheme
known as the \emph{JKO scheme} for solving PDEs and optimization
problems in $2$-Wasserstein space. Remarkably, the JKO scheme takes
the form of a regularized gradient descent algorithm of a tractable
loss, and whose gradients can be estimated from samples without
bias. The algorithm is the following generalized minimizing movements
\citep{de1993new} scheme:
\begin{equation}
  \label{eq:minimizing-movements:appendix}\tag{JKO}
  \varrho_{k+1} \in \arg\min_{\varrho}\left\{\kl{\varrho}{\mu} +
    \frac{1}{2\tau}\xxW_2^2(\varrho, \varrho_k)\right\}
\end{equation}
where $\mu(x)\propto\exp(-F_t(x))$, $f_t = \nabla F_t$, $\tau>0$ is the discretized
timestep, and $\varrho_{k}$ is short for $\varrho_{k\tau}$.

The JKO scheme has made several appearances in the ML
literature. Entropically-regularized optimal transport methods, for
instance, \citep{cuturi2013sinkhorn} are founded on the JKO
scheme. \citet{chizat2018global} uses the JKO scheme to guarantee
convergence to a global optimum when training neural networks without
convexity assumptions. In the RL literature, \citet{Zhang2018PolicyOA}
employs a JKO scheme to learn a posterior distribution over optimal
policies. More akin to the developments in this paper,
\citet{martin2020stochastically} presents a novel DRL algorithm where
the return distributions are trained as a WGF.

\section{Tempered Distributions}\label{app:distributions}
A recurring concept in many areas of mathematics, physics, and engineering is
that of \emph{generalized functions}, known as \emph{distributions}\footnote{Not
  to be confused with probability distributions.}. One such example is the Dirac
  delta. Distributions are particularly helpful at formally describing weakened solutions
  to PDEs by objects that may not be functions.

In this text, we will make use of the class of \emph{tempered} distributions,
which will be defined shortly. For more details, refer to
\citet{lax2002functional}.

\begin{definition}[Schwartz Class]
  Let $X$ be a normed space. A \emph{Schwartz class} is a class $\mathcal{S}$ of rapidly decaying-smooth
  functions,

  \begin{align*}
    \mathcal{S} = \left\{f\in C^\infty(X; \mathbf{R}) : \sup_{x\in X}(1 +
    \|x\|^k)|f^{(m)}(x)|<\infty\quad\forall k,m\in\mathbf{N}\right\}
  \end{align*}
\end{definition}

\begin{definition}[Tempered Distribution]\label{def:tempered-distribution}
  A tempered distribution is an element of the topological dual\footnote{The
    dual of a normed space is the set of all continuous, linear functionals on
    that space.} $\mathcal{S}'$
  of the Schwartz class $\mathcal{S}$.
\end{definition}

\begin{remark}
  The Dirac delta is the operator $\delta$ such that $\langle\delta, \phi\rangle
  = \phi(0)$. Clearly $\delta$ is linear, and since it is bounded,
  it is continuous. Therefore $\delta$ is indeed a tempered distribution.
\end{remark}

Tempered distributions admit a notion of differentiability, which can be used to
define ``distributional" solutions to PDEs.

\begin{definition}[Distributional
  Derivative]\label{def:distributional-derivative}
  Let $\mathcal{S}$ be a Schwartz class and $\psi\in\mathcal{S}'$ a tempered
  distribution. Then $\psi$ has a distributional derivative if there exists a
  tempered distribution $\psi'$ for which

  \begin{align*}
    \langle \psi', \phi\rangle &= -\langle\psi,
    \phi'\rangle\qquad\forall\phi\in\mathcal{S},
  \end{align*}

  and $\psi'$ is called the distributional derivative of $\psi$.
\end{definition}

\begin{definition}[Distributional Solutions of Hamilton-Jacobi
  PDEs]\label{def:distributional-solution}
  Consider the following PDE,
  \begin{equation}\label{eq:distributional-solution:pde}
    \partialderiv{u}{t} = f\circ u + \langle\nabla u, g\rangle +
    \quadraticform{h}{\hessian{y}u}
  \end{equation}

  where $u\in C^2(\mathbf{R}_+\times\mathcal{Y};\mathbf{R})$ for a normed space
  $\mathcal{Y}$.

  Then $\psi\in\mathcal{S}'$ is said to be a \emph{distributional solution} to
  \eqref{eq:distributional-solution:pde} if

  \begin{align*}
    &\int_0^\infty\int_{\mathcal{Y}}\phi(t, y)\left(f(\psi(y)) -
      \partialderiv{}{t}\psi(y)\right)dydt\\
    &\qquad=\int_{0}^\infty\int_{\mathcal{Y}}\bigg[\langle \psi(y)g(y), \nabla_y\phi(t,
    y)\rangle - \quadraticform{h(y)}{\psi(y)\hessian{y}\phi(t,
    y)}\bigg]dydt
  \end{align*}

  for every test function $\phi\in\mathcal{S}$. This is justified by simply multiplying both sides of
  \eqref{eq:distributional-solution:pde} by the test function, integrating over
  $\mathbf{R}_+\times\mathcal{Y}$, and substituting gradient terms of $\psi$
  with respect to its distributional derivative.
\end{definition}

\end{document}